\documentclass{article}
\usepackage{iclr2024_conference,times}

\usepackage{amsmath,amsfonts,bm}

\def\eqref#1{equation~\ref{#1}}

\def\1{\bm{1}}

\DeclareMathAlphabet{\mathsfit}{\encodingdefault}{\sfdefault}{m}{sl}
\SetMathAlphabet{\mathsfit}{bold}{\encodingdefault}{\sfdefault}{bx}{n}

\newcommand{\E}{\mathbb{E}}

\newcommand{\R}{\mathbb{R}}

\iclrfinalcopy

\usepackage{hyperref}
\usepackage{url}

\usepackage[utf8]{inputenc} 
\usepackage[T1]{fontenc}    
\usepackage{hyperref}       
\usepackage{url}            
\usepackage{booktabs}       
\usepackage{amsfonts}       
\usepackage{nicefrac}       
\usepackage{microtype}      
\usepackage{xcolor}         

\usepackage{amsmath}
\usepackage{amssymb}
\usepackage{mathtools}
\usepackage{amsthm}
\usepackage[algo2e,ruled,vlined,linesnumbered]{algorithm2e}
\theoremstyle{plain}
\newtheorem{theorem}{Theorem}[section]
\newtheorem{proposition}[theorem]{Proposition}
\newtheorem{lemma}[theorem]{Lemma}

\theoremstyle{definition}

\newtheorem{observation}[theorem]{Observation}
\theoremstyle{remark}

\DeclareMathOperator{\wf}{wf}
\DeclareMathOperator{\cost}{cost}
\DeclareMathOperator{\rank}{rank}
\DeclareMathOperator{\follower}{Follower}
\DeclareMathOperator{\robust}{Robust}
\DeclareMathOperator{\OPT}{OPT}
\DeclareMathOperator{\ALG}{ALG}
\DeclareMathOperator{\OFF}{OFF}
\DeclareMathOperator{\belady}{Belady}
\DeclareMathOperator{\FR}{F\&R}
\newcommand{\N}{\mathbb{N}}

\usepackage[textsize=tiny]{todonotes}
\usepackage{comment}
\definecolor{beaublue}{rgb}{0.74, 0.83, 0.9}

\def\FULL{}
\ifdefined\FULL
\excludecomment{shortversion}
\includecomment{fullversion}
\else
\includecomment{shortversion}
\excludecomment{fullversion}
\fi

\begin{shortversion}

\end{shortversion}

\begin{fullversion}

\end{fullversion}

\begin{shortversion}
\title{Algorithms for Caching and MTS with reduced number of predictions%
\thanks{Full version of this paper can be found in Appendix and at \url{https://arxiv.org/abs/??}}
}
\end{shortversion}
\begin{fullversion}
\title{
Algorithms for Caching and MTS\\ with reduced number of predictions}
\end{fullversion}

\author{%
Karim Abdel Sadek
\\
University of Amsterdam%
\thanks{The presentation of this paper  was financially supported
by the Amsterdam ELLIS Unit and Qualcomm.
Work completed while Abdel Sadek was in his final year of BSc at Bocconi University}%
\\
\texttt{karim.abdel.sadek@student.uva.nl}
\And
Marek Eliáš\\
Department of Computing Sciences\\
Bocconi University\\
\texttt{marek.elias@unibocconi.it}
}

\begin{document}

\maketitle

\begin{abstract}
ML-augmented algorithms utilize predictions to achieve performance
beyond their worst-case bounds.
Producing these predictions might be a costly operation
-- this motivated \citet{Im0PP22} to introduce the study of algorithms which use
predictions parsimoniously.
We design parsimonious algorithms for caching and MTS with {\em action
predictions}, proposed by \citet{AntoniadisCE0S20},
focusing on the parameters of consistency (performance with perfect
predictions) and smoothness (dependence of their performance on the prediction
error).
Our algorithm for caching is $1$-consistent, robust, and its smoothness
deteriorates with the decreasing number of available predictions.
We propose an algorithm for general MTS whose consistency and smoothness
both scale linearly with the decreasing number of predictions.
Without the restriction on the number of available predictions, both algorithms
match the earlier guarantees achieved by \citet{AntoniadisCE0S20}.
\end{abstract}

\section{Introduction%
}
\label{sec:intro}

Caching, introduced by \citet{ST85}, is a fundamental problem in online
computation important both in theory and practice.
Here, we have a fast memory (cache) which can contain up to $k$ different pages
and we receive a sequence of requests to pages in an online manner.
Whenever a page is requested, it needs to be loaded in the cache.
Therefore, if the requested page is already in the cache, it can be
accessed at no cost. Otherwise, we suffer a {\em page fault}:
we have to evict one page from the cache and load the requested page in its
place. The page to evict is to be chosen without knowledge of the future requests
and our target is to minimize the total number of page faults.

Caching is a special case of Metrical Task Systems introduced by
\citet{BorodinLS92} as a generalization of
many fundamental online problems.
In the beginning, we are given a metric space $M$ of states which can be
interpreted as actions or configurations of some system.
We start at a predefined state $x_0\in M$.
At time steps $t=1, 2, \dotsc$, we receive a cost function
$\ell_t\colon M\to \mathbb{R}^+\cup\{0, +\infty\}$
and we need to make a decision: either to stay at $x_{t-1}$ and pay a cost
$\ell_t(x_{t-1})$, or to move to another, possibly cheaper state
$x_t$ and pay $\ell_t(x_t) + d(x_{t-1},x_t)$, where the distance
$d(x_{t-1},x_t)$ represents the transition cost between states $x_{t-1}$ and
$x_t$.

The online nature of both caching and MTS forces an algorithm to make decisions
without knowledge of the future which leads to very suboptimal
results in the worst case \citep{BorodinLS92,ST85}.
A recently emerging field of learning-augmented algorithms,
introduced in seminal papers by \citet{KraskaBCDP18} and \citet{LykourisV21},
investigates approaches
to improve the performance of algorithms using predictions, possibly generated
by some ML model.
In general, no guarantee on the accuracy of these predictions is assumed.
Therefore, the performance of learning-augmented algorithms is usually evaluated using the following three parameters:

{\em Consistency.} Performance with perfect predictions, preferably close to
optimum.

{\em Robustness.} Performance with very bad predictions, preferably no worse than what is achievable by known algorithms which do not utilize predictions.

{\em Smoothness.} Algorithm's performance should deteriorate smoothly
with increasing prediction error between the consistency and robustness bound.

These three parameters express a desire to design algorithms that work very well when receiving reasonably accurate predictions most of the time and,
in the rest of the cases, still satisfy state-of-the-art worst-case guarantees.
See the survey by \citet{MV20} for more information.

Producing predictions is often a computationally intensive task,
therefore it is interesting to understand the interplay between the number
of available predictions and the achievable performance.
In their inspiring work, \citet{Im0PP22} initiated the study of 
learning-augmented algorithms which use the predictions parsimoniously.
In their work, they study caching with next-arrival-time predictions
introduced by \citet{LykourisV21}.
Their algorithm uses $O(b\log_{b+1} k)\OPT$ predictions, where $\OPT$ is
the number of page faults incurred by the offline optimal solution
and $b \in \N$ is a parameter. It 
achieves smoothness linear in the prediction error.
It satisfies tight consistency bounds:
with perfect predictions, it incurs at most
$O(\log_{b+1}k)\OPT$ page faults and no algorithm can do better.
In other words, it achieves a constant competitive ratio
with unrestricted access to predictions ($b=k$) 
and, with $b$ a small constant, its competitive ratio deteriorates to $O(\log k)$ 
which is comparable to the best competitive ratio achievable without
predictions.
One of their open questions is whether a similar result could be proved
for MTS.

In this paper, we study parsimonious algorithms for MTS
working with {\em action predictions} which were introduced by
\citet{AntoniadisCE0S20}.
Here, each prediction describes the state of an
optimal algorithm at the given time step and its error is defined as the distance
from the actual state of the optimal algorithm.
The total prediction error is the sum of errors of the individual predictions.
In the case of caching, action predictions
have a very concise representation, see Section~\ref{sec:Prelim_action}.
Unlike next-arrival-time predictions, action predictions can be used
for any MTS. Using the method of
\citet{BlumB00}, it is easy to achieve near-optimal
robustness for any MTS losing only a factor $(1+\epsilon)$ in consistency and smoothness. Therefore, we study how the reduced number
of predictions affects the consistency and smoothness parameters.
We consider the following two regimes.

{\em Bounded number of predictions:}
The algorithm can request a prediction whenever it prefers as 
far as the total number of requested predictions is bounded by
$b\OPT$, where $b$ is a parameter.
This regime is similar to \citet{Im0PP22}.

{\em Well-separated queries to the predictor:}
The queries to the predictor need to be separated by at least $a$ time
steps, for some parameter $a$.
This captures the situation when producing each prediction takes more than one
time step.

\subsection{Our results}

We evaluate the algorithm's performance using {\em competitive ratio}
which is, roughly speaking, the worst-case ratio between the cost
incurred by the algorithm and the cost of the offline optimum,
see Section~\ref{sec:Prelim} for a formal definition.
We say that an algorithm achieves consistency $\alpha$ and robustness $\beta$
if its competitive ratio is at most $\alpha$ when provided with perfect
predictions and at most $\beta$ with arbitrarily bad predictions.
For a given function $g$, we call
an algorithm $g(\eta)$-smooth
if its competitive ratio is at most $g(\eta)$ whenever provided with
predictions with the total error at most $\eta$.

Our first contribution is an algorithm for caching which receives
action predictions describing the states of the optimal offline algorithm
$\belady$ proposed by \citet{Belady66}.
High quality
such predictor based on imitation learning was already designed by
\citet{LiuHSRA20}. Its empirical evaluation within existing algorithms
designed for action predictions was performed by
\citet{ChledowskiSZ21}.

\begin{theorem}
\label{thm:FnR}
Let $f$ be an increasing convex function such that $f(0)=0$ and $f(i)\leq
2^i-1$ for each $i\geq 0$.
There is an algorithm for caching requiring
$O(f(\log k))\OPT$ predictions which achieves consistency $1$,
robustness $O(\log k)$, and smoothness
$O(f^{-1}(\eta/OPT))$, where $\eta$ denotes the total prediction error
with respect to $\belady$ and $\OPT$ is the number of page faults of $\belady$.
\end{theorem}

In fact, the number of required predictions is slightly smaller than
what is stated in the theorem.
Table~\ref{tab:f_smooth} shows numbers of predictions
and achieved smoothness for some natural choices of $f$.
Already with $O(\sqrt{k})\OPT$ predictions, our bounds are comparable to
\citet{AntoniadisCE0S20} whose algorithm asks for
a prediction in every step, its consistency is constant and
its smoothness is logarithmic in $\eta$.
The algorithm also works with $f(i)=0$. In that case, it asks
for at most $2\OPT$ predictions and still remains $1$-consistent.
However, its smoothness is not very good.
We use sliding marking phases and a careful distribution of queries of the
predictor over the time horizon.
This allows us to avoid dealing with so called "ancient" pages
considered by \citet{Rohatgi20} and \citet{AntoniadisCE0S20}, resulting
in an algorithm with better consistency and a simpler analysis.

\begin{fullversion}
\begin{table}[t]
\caption{Smoothness vs. number of predictions.}
\label{tab:f_smooth}
\hfill
\begin{tabular}{lcc}
\toprule
$f(i)$ & \# of predictions & smoothness\\
\midrule
$2^i-1$ & $O(\sqrt{k}) \OPT$ & $O(1+\log(\frac\eta\OPT + 1))$\\[1ex]
$i^2$ & $O(\log^2 k) \OPT$ & $O(\sqrt{2\frac\eta\OPT})$\\[1ex]
$i$ & $O(\log k) \OPT$ & $O(\frac\eta\OPT)$\\[.5ex]
0& $2\OPT$ & $O(\frac{k\eta}{\OPT})$\\[.5ex]
\bottomrule
\end{tabular}
\hfill\ %
\end{table}
\end{fullversion}
\begin{shortversion}
\begin{table}[t]
\caption{Smoothness vs. number of predictions.}
\label{tab:f_smooth}
\hfill
\begin{tabular}{lcccc}
\toprule
$f(i)$ & $2^i-1$ & $i^2$ & $i$ & 0\\
\midrule
\# of predictions  & $O(\sqrt{k}) \OPT$ & $O(\log^2 k) \OPT$ & $O(\log k) \OPT$ & $2\OPT$\\
smoothness & $O(1+\log(\frac\eta\OPT + 1))$ & $O(\sqrt{2\frac\eta\OPT})$ & $O(\frac\eta\OPT)$ & $O(\frac{k\eta}{\OPT})$\\
\bottomrule
\end{tabular}
\hfill\ %
\end{table}
\end{shortversion}

We discuss tightness of our bounds in
\begin{fullversion}
Section~\ref{sec:LBs}.
\end{fullversion}
\begin{shortversion}
Section 7 in the full version of this paper (see Appendix).
\end{shortversion}
We show that with, for example, only $0.5 OPT$ available
predictions, no algorithm can be better than $O(\log k)$-competitive
-- a guarantee comparable to the best classical online algorithms without
predictions.
We also show that the number of predictions used by our algorithm is close
to optimal.

\begin{theorem}
\label{thm:FnR_LB}
Let $f$ be an increasing function. Any $f(\eta)$-smooth algorithm for caching with action predictions, i.e., an algorithm
whose competitive ratio with predictions of error $\eta$ is $f^{-1}(\eta)$ for
any $\eta>0$, has to use at least $f(\ln k)\OPT$ predictions.
\end{theorem}

For general MTS, we cannot bound the number of used predictions as
a function of $\OPT$. The reason is that any instance of MTS
can be scaled to make $\OPT$ arbitrarily small, allowing us to use only
very few predictions.
We propose an algorithm which
queries the predictor once in every $a$ time steps, making at most
$T/a$ queries in total, where $T$ denotes the length of the input sequence.

\begin{theorem}
\label{thm:mts}
There is a deterministic algorithm for any MTS which receives a prediction only
once per each $a$ time steps and its cost is at most
$O(a) \cdot (\OFF + 2\eta)$,
where $\OFF$ denotes the cost of an arbitrary offline algorithm and $\eta$ the
error of predictions with respect to this algorithm.
\end{theorem}

This is a more general statement than Theorem~\ref{thm:FnR} which requires
$\OFF$ to be $\belady$.
Considering any offline optimal algorithm $\OFF$,
Theorem~\ref{thm:mts} implies a smoothness $O(a)\cdot(1+2\eta/\OPT)$
and consistency $O(a)$.
Our algorithm is based on work functions.
For $a=1$, its smoothness is $1+2\eta/\OFF$, see Section~\ref{sec:once_in_a},
which improves upon the smoothness bound of $1+4\eta/\OFF$ by
\citet{AntoniadisCE0S20}.
It is not robust on its own. However, it can be combined with any
online algorithm for the given MTS using the result of
\citet{BlumB00} achieving robustness comparable to that algorithm and losing
only a factor of $(1+\epsilon)$ in smoothness and consistency.

No algorithm receiving a prediction only once in $a$ time steps can be
$o(a)$-consistent. This follows from the work
of \citet{EmekFKR09} on advice complexity, see
\begin{fullversion}
Section~\ref{sec:LBs}
\end{fullversion}
\begin{shortversion}
Section 7 of the full version (in Appendix)
\end{shortversion}
for more details.
The same can be shown for smoothness by modifying the lower bound construction
of \citet{AntoniadisCE0S20}.

\begin{theorem}
\label{thm:mts_LB}
There is no $o(a\eta/\OPT)$-smooth algorithm for MTS with action predictions
which receives predictions only once in $a$ time steps.
\end{theorem}

We can modify our algorithm for caching to ensure that the moments
when the predictions are queried are separated by at least $a$ time steps,
not losing too much of its performance.
\begin{theorem}
\label{thm:FnR_a}
There is an algorithm for caching which receives prediction at most once
in $a\leq k$ time steps and using at most $O(f(\log k)) \OPT$ predictions
in total which is $O(1)$-consistent, $O(\log k)$-robust
and $O( f^{-1}(a\eta/\OPT))$-smooth.
\end{theorem}

In Section~\ref{sec:Experiments}, we provide empirical results suggesting that our
algorithm's performance can be comparable to the performance of algorithms
imposing no limitations on their use of predictions.
Our algorithm may therefore be useful especially with
heavy-weight predictors like \citep{LiuHSRA20}.

In Section~\ref{sec:FitF_oracle},
we provide an algorithm for an alternative prediction
setup which we call FitF oracle: each prediction says which of the pages in the current
algorithms cache will be requested furthest in the future.

\subsection{Related work}
The most related work is by \citet{Im0PP22}, who studied caching
with next arrival time predictions.
A smaller number of predictions affects the consistency of their algorithm: with
$b(\log k/\log b)\OPT$ predictions, they achieve consistency
$O(\log k/\log b)$ and they show that this is tight. They also show
that their algorithm achieves linear smoothness.
In contrast, our algorithm is $1$-consistent when receiving at least $\OPT$
predictions.
This demonstrates that action predictions,
although not containing more bits, seem to contain
useful information about the input instance in a more condensed form.
See \citep{AntoniadisCE0S20} for comparison and connections between these
prediction setups.
\citet{DrygalaNS23} study ski rental and bahncard problems with
predictions of a fixed cost.

There are several other papers on caching with predictions, including
\citep{LykourisV21,Rohatgi20,Wei20,EmekFKR09,AntoniadisCE0S20,AntoniadisBEFHLPS22} which
design algorithms asking for a prediction at each time step.
Consistency parameters achieved by these algorithms are constants greater than
$1$. Note that
those using black-box methods to achieve robustness are
$(1+\epsilon)$-consistent (e.g. \citep{Wei20}).
We can explicitly compare our smoothness bounds to \citet{AntoniadisCE0S20}
who use the same kind of predictions: their smoothness is
$O(1+\log(\frac{\eta}{\OPT}+1))$ with unlimited use of predictions while
our algorithm achieves the same smoothness bound with $O(\sqrt{k})\OPT$
predictions.
We compare the smoothness of the other algorithms experimentally in
Section~\ref{sec:Experiments}.
\citet{AntoniadisBEFHLPS22} study a prediction setup where each prediction
is only a single bit, however their algorithms need to receive
it in every time step.
\citet{GuptaPSS22} study several problems including caching in a setting
where each prediction is correct with
a constant probability.

\citet{AntoniadisCE0S20} proposed a $1$-consistent and
$(1+4\eta/\OPT)$-smooth algorithm for MTS with action predictions
which can be robustified by loosing factor $(1+\epsilon)$ in consistency and
smoothness.
Getting smoothness bounds sublinear in $\eta/\OPT$ for specific MTS problems
other than caching remains a challenging open problem even with unlimited
number of predictions and this holds even for weighted caching.
Specific results on weighted caching are by \citet{JiangPS22}
who studied it in a setup requiring very verbose predictions
and by \citet{BansalCKPV22} whose bounds depend on the
number of weight classes.
There is also a consistency/robustness trade-off by \citet{LindermayrMS22}
for $k$-server.

Since the seminal papers by \citet{KraskaBCDP18} and \citet{LykourisV21}
which initiated the study of learning-augmented algorithms,
many computational problems were considered.
\begin{shortversion}
There are papers on
ski rental~\citep{PurohitSK18},
secretary problem~\citep{DuttingLLV21},
online TSP \citep{BerardiniLMMSS22},
energy efficient scheduling \citep{BamasMRS20},
flow-time scheduling \citep{AzarLT21,AzarLT22},
and online page migration \citep{IndykMMR22a}.
Further related works can be found in References and are
discussed in the full version of this paper (see Appendix).
\end{shortversion}
\nocite{BFKLM17,DKP09,HKK10,BKKKM17}
\nocite{PurohitSK18,AntoniadisCEPS21,LindermayrM22,ImKQP21,Zeynali0HW21,BoyarFL22,DuttingLLV21,AntoniadisGKK20,EberleLMNS22,BamasMRS20,AntoniadisGS22,DinitzILMV21,chen22v,ErgunFSWZ22,SO22,PZ22,website}

\begin{fullversion}
There are papers on
ski rental~\citep{PurohitSK18,AntoniadisCEPS21},
secretary and matching problems~\citep{DuttingLLV21,AntoniadisGKK20},
online-knapsack~\citep{ImKQP21,Zeynali0HW21,BoyarFL22},
graph exploration~\citep{EberleLMNS22},
online TSP \citep{BerardiniLMMSS22},
energy efficient scheduling \citep{BamasMRS20},
flow-time scheduling \citep{AzarLT21,AzarLT22},
restricted assignment \citep{LattanziLMV20},
non-clairvoyant scheduling~\cite{PurohitSK18,LindermayrM22},
and online page migration \citep{IndykMMR22a}.
In offline setting, there is a work of
Dinitz et al.~\citep{DinitzILMV21} on matching,
Chen et al.~\citep{chen22v} on graph algorithms,
Polak and Zub~\citep{PZ22} on flows,
Sakaue and Oki~\citep{SO22} on discrete optimization problems,
and Ergun et al.~\citep{ErgunFSWZ22} on clustering.
We also refer to~\citep{website} to an updated list of results
in the area.

There are numerous works on advice complexity of online problems, where
algorithms are given certain number of bits of information about the future
which are guaranteed to be correct. This is different from our setting, where
we receive predictions of unknown quality.
We refer to the survey by \citet{BFKLM17}, work of
\citet{DKP09} on caching, \citet{EmekFKR09} on MTS, and further papers by
\citet{HKK10,BKKKM17}.

There are already works on predictors for caching.
\citet{Hawkey16} proposed a binary classifier called Hawkey which
predicts which pages will be kept in cache by $\belady$, providing
us with action predictions.
Their result was later improved by \citet{Glider19} who designed
a model called Glider for the same classification problem.
There is a very precise model
by \citet{LiuHSRA20} whose main output can be interpreted as an action
prediction although it has a second prediction head which produces
next arrival predictions. This model is large and relatively slow
and served as a motivation for this work.
\citet{ChledowskiSZ21} evaluated the performance of existing ML-augmented
algorithms with this predictor.
\end{fullversion}

\section{Preliminaries}
\label{sec:Prelim}

Consider an algorithm $\ALG$ for MTS which produces a solution
$x_0, x_1, \dotsc, x_T$
for an instance $I$ consisting of cost functions $\ell_1, \dotsc, \ell_T$.
We denote $\cost(\ALG(I))= \sum_{t=1}^T (\ell_t(x_t) + d(x_{t-1},x_t))$.
We say that $\ALG$ is $r$-competitive with respect to an
offline algorithm $\OFF$ if there is an absolute constant $\alpha \in \R$ such
that $\E[\cost(\ALG(I))] \leq r\cdot \cost(\OFF(I)) + \alpha$
for any instance $I$. If $\ALG$ is $r$-competitive with respect to an optimal
offline algorithm, we say that $\ALG$ is $r$-competitive and call $r$ the
competitive ratio of $\ALG$.
In the classical setting (without predictions), the best achievable competitive
ratios are $\Theta(\log k)$ for caching \citep{FiatKLMSY91} and of order
$poly\log n$ for MTS \citep{BBM06,BCLL19},
where
$n$ is the number of points in the underlying metric space $M$.
We refer to \citep{BEY98} for a textbook treatment.

\subsection{Action predictions for MTS}
\label{sec:Prelim_action}
\citet{AntoniadisCE0S20} proposed a prediction setup which they call
{\em action predictions}, where the predictions tell us what a good algorithm
would do. More precisely, at each time $t$, the algorithm receives a prediction
$p_t$ of a state where some offline algorithm $\OFF$ moves to at time $t$.
The error of prediction $p_t$ is then
$\eta_t = d(p_t, o_t)$, where $o_t$ is the real state of $\OFF$ at time $t$.
The total prediction error is defined as $\eta = \sum_{t=1}^T \eta_t$.

Considering the case of caching, the state corresponds to a cache content, and the prediction
error is the number of pages present in the cache of $\OFF$ and absent
from the predicted cache content.
A whole cache content may seem like a huge piece of information, but
action predictions for caching can be implemented in a very succinct way.
\citet{AntoniadisCE0S20} explain how to represent them with only $O(\log k)$ bits
per time step when they are received at each time step.
Our algorithm asks, in each query, a specific number of indices of pages which
are present in its cache but absent from the predicted cache.
When we talk about a bound on the number of provided predictions,
this bound applies both to the number of such queries as well as
to the total number of indices reported by the predictor during the running time
of the algorithm.
There are predictors which can generate predictions of a similar kind by
\citet{Hawkey16,Glider19,LiuHSRA20}. See \citep{AntoniadisCE0S20} for
a detailed treatment of this prediction setup and a comparison to other setups
for caching.

\subsection{Caching: Belady's algorithm, Marking, and Lazy algorithms}
The classical optimal offline algorithm for caching proposed
by \citet{Belady66} is denoted $\belady$ in this paper.
At each page fault, it evicts a page which is going to be requested furthest
in the future. In the case of a tie, i.e., if there are several pages in the cache
which will not be requested anymore, it chooses one of them arbitrarily.
Our caching algorithm assumes that the predictor is trying
to simulate $\belady$.
The following useful property allows us to detect
errors in the predictions quickly.
It was recently used by \citet{EliasKMM24}.
\begin{observation}
\label{lem:FitF_monot}
Consider request sequence $r_1, \dotsc, r_T$. For any $t\leq T$, the cost
incurred by $\belady$ for $r_1, \dotsc, r_T$ until time $t$ is the same
as the cost of $\belady$ with input $r_1, \dotsc, r_t$.
\end{observation}
To see this, it is enough to note that the solution produced by
$\belady$ with input $r_1, \dotsc, r_T$ agrees
until time $t$ with the solution produced by
$\belady$ on $r_1, \dotsc, r_t$ which breaks ties based
on the arrival times in $r_{t+1}, \dotsc, r_T$.

We use properties of {\em marking} algorithms in this work.
Such algorithms split the input sequence into phases, i.e., maximal
subsequences where at most $k$ distinct pages are requested.
Usually, the first phase starts in the beginning and the next phase
follows just after the end of the previous one. However, we will consider
phases starting at arbitrary moments.
Let $O$ be the cache content of an algorithm in the beginning of the phase.
Whenever a page is requested for the first time during the phase,
we call this moment an {\em arrival}
and we {\em mark} the page. At the end of the
phase, the set $M$ of marked pages will have size $k$:
some of them belong to $O$ and are called {\em old}
while those in $C=M\setminus O$ are called {\em clean}.
Exactly $|C|$ pages from $O$ remain unmarked until the end of the phase.

Marking algorithms is a class of algorithms which never evict a marked page
and all of them have cache content $M$ at the end of the phase.
$\belady$ is not marking and our algorithm is not marking either,
although it uses ideas from marking to achieve desired robustness and
smoothness properties. At the end of each phase, we can bound the
difference between the cache content of some algorithm and marking.

\begin{observation}
\label{obs:marking_diverg}
Let $c$ be the cost incurred by some algorithm during a marking phase.
Then, $c\geq |M\setminus S|$, where $S$ is the cache content
of the algorithm at the end of the phase
and $M$ is the set of $k$ pages requested during the phase.
\end{observation}

This is because
each page in $p\in M$ has to be present in algorithm's cache
when requested during the phase. If $p \notin S$, then
the algorithm must have evicted it during the phase incurring cost $1$.

\begin{observation}
\label{obs:marking_FitF_err}
If a page $p$ is evicted by $\belady$ at time $t$, then
$p$ is not going to be requested in the marking phase containing $t$ anymore.
\end{observation}

If $p$ is evicted by $\belady$ at time $t$, then
the currently requested page $r_t$ and $k-1$ pages from the
cache are $k$ distinct pages that are requested before the moment when $p$ is
requested next time. The current phase then needs to end before that moment.

We say that an algorithm is {\em lazy} if it evicts only one page at a time
and only at a page
fault. $\belady$ is lazy while our algorithm, as described, may
not be. However, any algorithm can be made lazy without increasing its cost.
See \citep{BEY98} for more information about caching.

\begin{observation}
\label{obs:lazy_diverg}
The difference in the cache content of two lazy algorithms can increase
only if both of them have a page fault. In that case, it can increase
by at most 1.
\end{observation}

\begin{fullversion}
\subsection{MTS and advice complexity}
Advice complexity studies the performance of algorithms depending
on the number of bits of precise information about the instance available
in advance. In the case of MTS, \citet{EmekFKR09} study the situation when
algorithm receives $\frac1a \log n$ bits of information about the
state of some optimal offline algorithm, being able to identify its true state
once in $a$ time steps.
They formulate the following proposition for $\OFF$ being an optimal algorithm,
but the proof does not use its optimality and it can be any algorithm
located at $q_i$ at time $ia$.

\begin{proposition}[\citet{EmekFKR09}]
\label{prop:emek}
There is an algorithm which, with knowledge of state $q_i$ of algorithm $\OFF$
at time $ia$ for $i=1, \dotsc, T/a$,
is $O(a)$-competitive w.r.t. $\OFF$.
\end{proposition}

In our context,
we can say that the algorithm from the preceding proposition is
$O(a)$-consistent
if $q_1, \dotsc, q_{T/a}$ are states of an optimal solution.
However, it is not smooth
because it may not be possible to relate the cost of $\OFF$ to the prediction
error with respect to $\OPT$.
\end{fullversion}

\section{Bounded number of predictions}
\label{sec:FnR}
In this section, we prove Theorem~\ref{thm:FnR}.
We propose an algorithm called $\FR$ which consists of two parts: $\follower$ and $\robust$.
It starts with $\follower$ which is $1$-consistent, but lacks in smoothness
and robustness.
At each page fault, $\follower$ recomputes $\belady$
for the part of the sequence seen so far and checks whether it also has
a page fault. If yes, it copies the behavior of the predictor
(Line~\ref{alg:foll_following}). Otherwise,
it must have received an incorrect prediction before. Therefore, it
switches to $\robust$ (Line~\ref{alg:foll_error})
which is no more $1$-consistent, but achieves
desired smoothness and robustness. $\robust$ runs for one marking phase
and then returns back to $\follower$.
At such moment, the predictor's and the algorithm's cache can be very different
and $\follower$ may need to lazily synchronize with the predictor
(Line~\ref{alg:foll_lazy-sync}).

\begin{algorithm2e}
\caption{$\follower$}
\label{alg:follower}
$P:=$ initial cache content\tcp*{Prediction for time 0}
\ForEach{pagefault}{
    \lIf{$r_t\notin P$ and $\belady$ has a pagefault}{
	\nllabel{alg:foll_following}
        query new prediction $P$ and evict any $p\in C\setminus P$
    } \lElseIf{$r_t\in P$}{
	\nllabel{alg:foll_lazy-sync}
        evict arbitrary $p\notin P$
    } \lElse{
	\nllabel{alg:foll_error}
        Switch to $\robust$ (Algorithm~\ref{alg:robust})
    }
}
\end{algorithm2e}

Algorithm $\robust$ runs during a single marking phase
starting at the same moment,
splitting it into windows as follows (assuming $k$ is a power of 2):
The first window $W_1$ starts at the beginning of the phase and lasts $k/2$
arrivals, i.e., it ends just before the arrival number $k/2+1$.
$W_i$ follows the $W_{i-1}$ and its length is half of the
remaining arrivals in the phase. The last window $W_{\log k +1} = \{k\}$
lasts until the end of the phase.
$\robust$ comes with an increasing convex function $f$ such that $f(0)=0$
and $f(i) \leq 2^j-1$.
Faster growing $f$ does not further improve our smoothness bounds.
Function $f$ determines that we should request $f(i)-f(i-1)$
predictions in window $i$.
If the window is too small, we ask for prediction at each time step.
$\robust$ starts with the cache content of a marking algorithm whose new phase
would start at the same moment (Line~\ref{alg:robust_M}).
In the case of a page fault, it evicts an unmarked page chosen uniformly at random.
At arrivals belonging to the sets $S$ and $F$,
it performs synchronization with the predictor
and queries the predictor's state respectively.
The synchronization is always performed with respect to the most recent prediction
$P$ which, in the case of lazy (or lazified) predictors, implicitly
incorporates information
from the previous predictions.

\begin{algorithm2e}
\caption{$\robust_f$ (one phase)}
\label{alg:robust}
\ShowLn\nllabel{alg:robust_M}
Load $k$ distinct most recently requested pages\;
$S:= \{k-2^j+1\mid j=\log k ,\dotsc, 0\}$\;
$W_i := [k-2^{\log k -i+1} + 1, k-2^{\log k -i}]$ for $i=1, \dotsc, \log k$
	and $W_{\log k + 1} = \{k\}$\;
Choose $F \subseteq \{1, \dotsc, k\}$ such that
        $|F \cap W_i| = \min\{f(i)-f(i-1), |W_i|\}$ for each $i$\;
\ForEach{pagefault during the phase}{
    \lIf{it is arrival belonging to $F$}{
        ask for new prediction $P$}
    \lIf{it is arrival belonging to $S$}{
	\nllabel{alg:robust_sync}
        synchronize with $P$}
    \lIf{requested page is still not in cache}{
        evict random unmarked page}
}
Load all pages marked during the phase\;
\nllabel{alg:robust_M_2}
Switch to $\follower$ (Algorithm~\ref{alg:follower})\;
\end{algorithm2e}

Synchronization with $P$ (Line~\ref{alg:robust_sync}) works as follows.
All pages previously evicted by random evictions
return to the cache
and the same number of pages not present in $P$ is evicted.
We denote $E_i = E_i^- \cup E_i^+$
the set of pages evicted at the beginning of $W_i$,
where pages in $E_i^-$ are requested during $W_i$ while
those in $E_i^+$ are not.
Note that algorithm's and predictor's cache may not become the same
after the synchronization.
Since the algorithm starts with pages in $M$ and loads only clean pages,
we have the following observation.

\begin{observation}
Let $C_i$, $|C_i|=c_i$ be the set of
clean pages arriving before the start of $W_i$.
Then, $E_i \subseteq M\cup C_i$ and $|E_i|= |M\cup C_i| - k = c_i$.
\end{observation}

We assume that the predictor is lazy and does not load pages that
are not requested. Therefore, no page from $E_i^+$ will be loaded during $W_i$
by the predictor and the same holds for $\robust$, implying the following.

\begin{observation}
\label{obs:lazy_predictor}
For every $i=1, \dotsc, \log k$,
we have
$E_i^+ \subseteq E_{i+1}$ and therefore $E_i\setminus E_{i+1}\subseteq E_i^-$.
\end{observation}

Synchronization with the marking cache performed by $\robust$
is to ensure that the difference between
the cache of the algorithm and $\belady$ can be bounded by costs incurred
by $\belady$ locally using Observation~\ref{obs:marking_diverg} instead
of diverging over time solely due to incorrect predictions.

\textbf{Implementation suggestions.}
Algorithms are described as to simplify the analysis.
Synchronization in $\robust$ (line~\ref{alg:robust_sync})
should be done lazily as to make use of the most recent prediction.
At arrivals of clean pages, one may evict a page not present in predictor's
cache instead of a random unmarked page;
one can also ask for a fresh prediction (at most $2\OPT$ additional
queries).
The second synchronization with the marking cache in $\robust$ (line~\ref{alg:robust_M_2}) can be omitted.
With $f(i)=0$, one can query the predictor only at clean
arrivals, using at most $2\OPT$ predictions in total.
We recommend a lazy implementation.
Since $\robust$ is not $1$-consistent, one may also switch from $\follower$
only once $\follower$'s cost is at least a constant (e.g. 2 or 3)
times higher than the cost of $\belady$.

We denote
$H_i$ the $i$th phase of $\robust_f$ and $H_i^-$ a hypothetical marking phase
which ends just before $H_i$ starts. Note that $H_i^-$ might overlap with
$H_{i-1}$.
But $H_1, H_2, \dotsc$ are disjoint and we denote $G_{i,i+1}$ the time interval
between the end of $H_i$ and the beginning of $H_{i+1}$.
$c(H_i)$ is the number of clean pages during phase $H_i$.
For a given time period $X$, we define
$\Delta^A(X)$, $\Delta^B(X)$, and $\Delta^P(X)$ the costs incurred by
$\FR$, $\belady$, and the predictor
respectively during $X$ and $\eta(X)$ the error of predictions
received during $X$.

Here is the main lemma about the performance of $\robust$.
\begin{fullversion}
Its proof
is deferred to Section~\ref{sec:robust}.
\end{fullversion}
\begin{shortversion}
Overview of its proof
is deferred to Section~\ref{sec:robust}.
\end{shortversion}
\begin{lemma}
\label{lem:robust}
Denote $X_i = H_{i-1}\cup H_i^- \cup H_i$.
During the phase $H_i$, $\robust_f$ receives at most $f(\log k)+1$ predictions
and we have
\begin{align}
\label{eq:robust_f}
\E[\Delta^A(H_i)] \leq
        O(1)f^{-1}\bigg(\frac{\eta(H_i)}{\Delta^B(X_i)}\bigg) \Delta^B(X_i).
\end{align}
At the same time, we also have
\begin{align}
\label{eq:robust_robustness}
\E[\Delta^A(H_i)] &\leq O(\log k) \Delta^B(X_i)\text{ and}\\
\label{eq:robust_regret}
\Delta^A(H_i) &\leq O(k) + O(k)\eta(H_i).
\end{align}
\end{lemma}

\begin{fullversion}
\subsection{Analysis of Follower}
\end{fullversion}

\begin{shortversion}
The following lemma is useful to analyze the cost incurred during
the $\follower$ part of the algorithm.
The proof of Theorem~\ref{thm:FnR} then combines it with
with Lemma~\ref{lem:robust} and
can be found in the full version of the paper.
\end{shortversion}

\begin{lemma}
\label{lem:gap}
Consider the gap $G_{i,i+1}$ between phases $H_i$ and $H_{i+1}$ of $\robust_f$.
We have
\[ \Delta^A(G_{i,i+1})
        \leq \Delta^B(G_{i,i+1}) + \Delta^B(H_i).\]
\end{lemma}
\begin{fullversion}
\begin{proof}
There are $\Delta^B(G_{i,i+1})$ page faults served at
line~\ref{alg:foll_following} because $\belady$ also has those page faults.
To bound the cost incurred on line~\ref{alg:foll_lazy-sync}, we
denote $P, B, M$ the cache contents of the predictor, $\belady$, and $\robust$
respectively at the end of the phase $H_i$.
The synchronization with $P$ costs at most
$|(P\setminus M)\cap B|$ if we omit costs incurred by $\belady$ which
were already accounted for above.
However, $(P\setminus M)\cap B = (B\setminus M)\cap P \subseteq B\setminus M$
and $|B\setminus M|\leq \Delta^B(H_i)$ by Observation~\ref{obs:marking_diverg}.
\end{proof}
\end{fullversion}
\begin{fullversion}
\begin{proof}[Proof of Theorem~\ref{thm:FnR}]
The cost of $\follower$ until the start of $H_1$ is the
same as the cost of $\belady$.
Therefore, 
by lemmas \ref{lem:gap} and \ref{lem:robust} \eqref{eq:robust_f},
the cost of $\FR$, in expectation, is at most
\begin{align*}
OPT &+ \sum_i \Delta^B(H_i)
        + \sum_i O(1)\Delta^B(X_i)
        f^{-1}\bigg(\frac{\eta(H_i)}{\Delta^B(X_i)}\bigg),
\end{align*}
where the sums are over all phases of Robust and $X_i = H_{i-1}\cup H_i^- \cup H_i$.
Since phases $H_i$ are disjoint and the same holds for $H_i^-$,
this expression is at most
$ \OPT \cdot O(f^{-1}(\eta/\OPT))$
by concavity of $f^{-1}$, implying the smoothness bound for $\FR$.

If we use bound \eqref{eq:robust_robustness} instead of \eqref{eq:robust_f},
we get $O(\log k)\OPT$ -- the robustness bound.
Since there must be at least one error during the execution of $\follower$
to trigger each execution of
$\robust$, \eqref{eq:robust_regret} implies that the cost of $\FR$ is at
most $\OPT + \eta O(k)$. With $\eta=0$, this implies $1$-consistency of $\FR$.
$\follower$ queries the predictor only at a page fault by $\OPT$
and the prediction consists of a single page evicted by the predictor.
$\robust$ may ask for up to $f(\log k)+1$ predictions in each phase,
each of them consisting of indices of at most $c(H_i)$ pages from $\FR$ cache
not present in the predictor's cache. This gives both
$O(f(\log k))\OPT$ queries to the predictor as well as $O(f(\log k))\OPT$ predicted
indices in total.
\end{proof}
\end{fullversion}

\begin{fullversion}
Note that $\robust_f$ can rarely use full $O(f(\log k))$ predictions,
because the last windows are not long enough. More precise calculation
of numbers of predictions can be found in
Appendix~\ref{app:pred_numbers}.
\end{fullversion}

\subsection{Analysis of $\robust_f$}
\label{sec:robust}

\begin{shortversion}
The full version of this section
and the proof of Lemma~\ref{lem:robust} can be found in
Appendix (Section 3.2), here we include a short overview.
Charging a page fault on a page evicted due to predictor's advice to a single
incorrect action prediction can only give us smoothness linear in
the prediction error. This is in contrast with next-arrival predictions where
algorithms can be analyzed by estimating lengths of eviction chains caused
by each incorrect prediction, as proposed by \citet{LykourisV21}.
To achieve sublinear smoothness, we need to charge each such page fault
to a long interval of incorrect predictions.
This is the most challenging part of our analysis because $\belady$ also moves
and the same prediction incorrect at one time step may be correct at another
time step. We estimate the error of predictions received during each window
by introducing window rank which bounds the prediction error
from below accounting for the movements of $\belady$.
\end{shortversion}

\begin{fullversion}
First, we relate the number of clean pages in a robust phase to the costs
incurred by $\belady$.

\begin{observation}
\label{lem:clean_pages}
Consider a phase $H$ denoting $C(H)$ the set of clean pages
arriving during $H$. We have
\[ c(H) := |C(H)| \leq \Delta^B(H^-) + \Delta^B(H). \]
\end{observation}
\begin{proof}
There are $k+c(H)$ pages requested during $H^- \cup H$. Therefore,
any algorithm, and $\belady$ in particular, has to pay cost
$\Delta^B(H^-\cup H) \geq c(H)$.
\end{proof}

\begin{lemma}
\label{lem:M_sync}
Consider phase $H_i$.
Cost incurred by $\robust$ for synchronization with marking in
Line~\ref{alg:robust_M} is at most $\Delta^B(H_{i-1}\cup H_i^-)$.
\end{lemma}
\begin{proof}
Let $M$ denote the $k$ distinct most recently requested pages -- these
are marked pages during $H_i^-$. We consider two cases.

If $H_i^- \cap H_{i-1} = \emptyset$, then whole $H_i^-$ was served
by $\follower$. Each $p\in M$ must have been in the cache of
both $\follower$ and $P$ when requested and $\follower$ would evict
it afterwards only if $P$ did the same. Therefore,
$\robust$ needs to load at most $\Delta^P(H_i^-)=\Delta^B(H_i^-)$ pages.

If $H_i^-$ and $H_{i-1}$ overlap, let $M'$ denote the set of pages
marked during $H_{i-1}$. At the end of $H_{i-1}$, $\robust$ loads
$M'$ and $\follower$ loads only pages from $M$ until the end of $H_i^-$.
Therefore, $\robust$ starting $H_i$ needs to evict only pages from
$M'\setminus M$.
Now, note that there are $|M'\cup M|$ distinct pages requested
during $H_{i-1}\cup H_i^-$ and therefore
$|M'\setminus M| = |M'\cup M| - k \leq \Delta^B(H_{i-1}\cup H_i^-)$.
\end{proof}

We consider costs incurred by $\robust$ during window $W_i$
for $i=1, \dotsc, \log k+1$.
Note that $E_1=E_1^+=E_1^-=\emptyset$, since $W_1$ starts at the beginning
of the phase and there are no clean pages arriving strictly before $W_1$.

\begin{lemma}
\label{lem:cost_per_window}
Expected cost incurred by $\robust_f$ during $W_1$ is at most $2c_2$.
For $i=2, \dotsc, \log k +1$, we have
\[
\E[\Delta^A(W_i)] \leq |E_{i-1}^-| + c_i-c_{i-1} + 2(|E_i^-| + c_{i+1}-c_i),
\]
denoting  $c_{\log k +2} = c(H)$.
\end{lemma}
\begin{proof}
First, consider the costs during $W_1$. There are $c_2-c_1 = c_2$ clean pages
arriving during $W_1$ and $\robust$ has a page fault due to each of them,
evicting a random unmarked page.
In the worst case, all these clean pages arrive at the beginning of $W_1$.
Therefore, at arrival $c_2+1$, there are $c_2$ pages evicted
which were chosen among unmarked pages uniformly at random.
Let $U_a$ denote the set of unmarked pages at arrival $a$.
We have $U_1 = M$ (the initial cache content of $\robust$) and none
of those pages get marked during first $c_2$ arrivals.
During every arrival $a=c_2+1, \dotsc, k/2= S[2]-1$,
a single unmarked page is marked and we have
$|U_a| = k - (a-c_2)$. As in the classical analysis of Marker
(see \citep{BEY98} and references therein),
the probability of the requested unmarked page
being evicted is $c_2/|U_a|$. We have
\[\Delta^A(W_1) = c_2 + \sum_{a=c_2+1}^{S[2]-1} \frac{c_2}{|U_a|}
\leq c_2 + \frac{k}2 \cdot \frac{c_2}{k/2} = 2c_2.
\]

For $i\geq 2$, there are $c_i$ pages evicted before the start of $W_i$:
those in $E_{i-1}^+$ were evicted due to synchronization with the predictor
and the rest were evicted in randomized evictions -- those are loaded back to
the cache at the beginning of $W_i$, causing cost
$c_i - |E_{i-1}^+| = |E_{i-1}^-| + c_{i}-c_{i-1}$.
After this synchronization, all unmarked pages are in the cache except
those belonging to $E_i$.

During $W_i$, pages from $E_i^-$ and $c_{i+1}-c_i$ new clean pages are requested
causing page faults which are resolved by evicting a random unmarked
page from the cache.
In the worst case, these page faults all happen in
the beginning of the window, leaving more time for page faults on
randomly-evicted pages.
Let $\bar a$ denote the first arrival after these page faults
and $U_{\bar a}$ the set of unmarked pages at that moment.
At arrival $\bar a$, there are $c_{i+1}$ pages missing from the cache:
pages from $E_i^+$ which are not going to be requested during $W_i$
and $|E_i^-| + c_{i+1}-c_i = c_{i+1} - |E_i^+|$
unmarked pages were chosen uniformly at random from
$U_{\bar a}\setminus E_i = U_{\bar a} \setminus E_i^+$.
This is because only pages which were marked since the beginning of $W_i$
until $\bar a$ are those from $E_i^-$ and they were not present in the
cache before they got marked.
We compute the expected number of page faults on the randomly evicted pages.
Since they are unmarked when evicted, such page faults
can happen only on arrivals.

At arrival $a$, the set of unmarked pages $U_a$ has size $k-(a-c_{i+1})$.
For any $a\in W_i$ such that $a\geq \bar a$, we have
$U_a \cap E_i^- = \emptyset$ and
pages in $U_a\cap E_i^+$ are evicted with probability 1.
So, $c_{i+1}-|E_i^+|$ evicted pages are picked uniformly at random from
$U_a\setminus E_i^+$ of size at least
$k-(a-c_{i+1})-|E_i^+|$.
Therefore, our expected cost is at most
\begin{align*}
\sum_{a=\bar a}^{S[i+1]-1} \frac{c_{i+1}-|E_i^+|}{k-a+c_{i+1}-|E_i^+|}
\leq \sum_{a\in W_i} \frac{c_{i+1}-|E_i^+|}{k-a}
\leq \frac{k}{2^i} \cdot \frac{c_{i+1}-|E_i^+|}{k/2^i}
= c_{i+1}-|E_i^+|
\end{align*}
which is equal to $|E_{i-1}^-| + c_{i}-c_{i-1}$.
Note that $k- a\geq k- S[i+1]-1 = 2^{\log k - i} = k/2^i$.
\end{proof}

For $i=1, \dotsc, \log k$,
we define
\[ \rank(W_i) := |E_{i+1} \cap B_{i+1}| - \Delta^B(W_i),\]
where $B_i$ denotes the cache content of $\belady$ at the beginning of $W_i$.
We do not define $\rank(W_i)$ for $i=1+\log k$.
Later, we relate rank to the prediction error. We have
the following lemma.

\begin{lemma}
\label{lem:cost_per_win_rank}
During a phase $H$, the expected cost of $\robust_f$ is at most
\[ 3\sum_{i=1}^{\log k+1} \rank(W_{i-1}) + 3c(H) + 6\Delta^B(H). \]
\end{lemma}
\begin{proof}
First, we observe that
\begin{equation}
\label{eq:rank_vs_e-}
|E_i^-| \leq \rank(W_{i-1}) + \Delta^B(W_{i-1}) + \Delta^B(W_i)
\end{equation}
holds for $i=2, \dotsc, \log k +1$.
This is because
$|E_i^- \cap B_i| \leq |E_i \cap B_i| = \rank(W_{i-1}) + \Delta^B(W_{i-1})$
and $|E_i^- \setminus B_i| \leq \Delta^B(W_i)$ due to pages from $E_i^-$
being requested during $W_i$ and $\belady$ having to load them.
Combining \eqref{eq:rank_vs_e-} with Lemma~\ref{lem:cost_per_window},
and summing over all windows, we get the statement of the lemma.
\end{proof}

Now, we relate the rank of a window to the prediction error.
\begin{lemma}
\label{lem:error_to_rank}
Denote $\eta_i$ the error of predictions arriving during $W_i$.
We have
\[ \eta(W_i) \geq |F\cap W_i| \rank(W_i). \]
\end{lemma}
\begin{proof}
Prediction error at time $t$ is $\eta_t = |B_t\setminus P_t|$.
At the end of $W_i$, it is at least $|E_{i+1}\cap B_{i+1}|$.
Due to laziness of the predictor,
$|B_t\setminus P_t|$
can increase only if both predictor and $\belady$ have
a page fault: in that case it may increase by 1,
see Observation~\ref{obs:lazy_diverg}.
Therefore, at any time $t$ during $W_i$,
we have
$\eta_t = |B_t\setminus P_t| \geq |E_{i+1}\cap F_{i+1}| - \Delta^B(W_i)
= \rank(W_i)$.
Since we query the predictor at arrivals belonging to $F$,
the total error of received predictions is at least
$|F\cap W_i|\rank(W_i)$.
\end{proof}

We will analyze intervals of windows starting when some particular
incorrect prediction was introduced and ending once it was corrected.
The following lemma charging the increase of rank to the arriving clean pages
and costs incurred by $\belady$
will be used to
bound the number of such intervals.

\begin{lemma}
\label{lem:rankdelta}
For $i=1, \dotsc, \log k$, we have
\[ \rank(W_i) - \rank(W_{i-1})
\leq \Delta^B(W_{i-1}) + c_{i+1} - c_i,
\]
denoting $W_0$ an empty window before $W_1$ with $\rank(W_0)=0$.
\end{lemma}
\begin{proof}
It is enough to show that
\[ |E_{i+1}\cap B_{i+1}| \leq |E_i \cap B_i| + \Delta^B(W_i) + c_{i+1} -c_i. \]

Since the right-hand side is always non-negative, we only need to consider the
case when the left-hand side is positive.
We show how to charge pages in $E_{i+1}\cap B_{i+1}$ either to pages
in $E_i\cap B_i$ or to $\Delta^B(W_i)$ and $c_{i+1}-c_i$.

Since $|E_i|=c_i \leq c_{i+1} = |E_{i+1}|$, we can construct an injective map
$\beta \colon E_i \to E_{i+1}$,
such that $\beta(p) = p$ for each $p\in E_i\cap E_{i+1}$.
There are $|E_{i+1}|-|E_i|=c_{i+1}-c_i$ pages $p$ such that
$\beta^{-1}(p)$ is not defined.
We show that, for each $p\in E_{i+1}\cap B_{i+1}$ for which it is defined,
$\beta^{-1}(p)$ is either a page in $E_i\cap B_i$
or a page loaded by $\belady$ during $W_i$.
There are two cases.
\begin{itemize}
\item $p\in E_i\cap E_{i+1}$ implying $\beta^{-1}(p) = p$.
Then either $p\in B_i$ and therefore $p \in E_i\cap B_i$,
or $p\in B_{i+1}\setminus B_i$
implying that $\belady$ has loaded $p$ during $W_i$.
\item $p\notin E_i\cap E_{i+1}$
        implying $q = \beta^{-1}(p) \in E_i\setminus E_{i+1}$.
By Observation~\ref{obs:lazy_predictor}, $q \in E_i^-$, i.e.,
it must have been requested during $W_i$.
If $q\in B_i$ then, $q\in E_i\cap B_i$. Otherwise,
$\belady$ must have loaded $q$ during $W_i$.
\end{itemize}
To sum up: $\beta$ is an injective map and $\beta^{-1}(p)$ does not exist
for at most $c_{i+1}-c_i$ pages $p \in E_{i+1}\cap B_{i+1}$.
All other $p \in E_{i+1}\cap B_{i+1}$ are mapped by $\beta^{-1}$ to a unique
page either belonging to $E_i\cap B_i$ or loaded by $\belady$ during $W_i$.
\end{proof}

\begin{proof}[Proof of Lemma~\ref{lem:robust}]
$\robust_f$ receives a prediction only at an arrival belonging to $F$.
Since $F$ contains
$\sum_{i=1}^{\log k + 1} |F\cap W_i| \leq \sum_{i=1}^{\log k} (f(i) - f(i-1))
+ |W_{\log k +1}| \leq f(\log k) + 1$
arrivals, because $|W_{\log k+1}| = 1$ and $f(0)=0$,
there are at most $f(\log k)+1$ queries to the predictor.

To prove equations
(\ref{eq:robust_f},\ref{eq:robust_robustness},\ref{eq:robust_regret}),
we always start with bounds proved in lemmas \ref{lem:clean_pages},
\ref{lem:cost_per_win_rank} and \ref{lem:M_sync}.
In the rest of the proof, we write $H$ instead of $H_i$
and $X$ instead of $H_i$ to simplify the notation.

To get equation \eqref{eq:robust_robustness}, note that $\rank(W_i)\leq c(H)$
for each window $i$. Therefore, by lemmas \ref{lem:clean_pages},
\ref{lem:cost_per_win_rank} and \ref{lem:M_sync}, we can bound
$\E[\Delta^A(H)]$ by
\[ \Delta^B(X) + 3c(H)\log k + 3c(H) + 6\Delta^B(H) \leq
	O(\log k)\Delta^B(X).\]

To get equation \eqref{eq:robust_regret}, note that $\Delta^B(H)\leq k$,
$c(H)\leq k$,
and $\rank(W_i) > 0$ only if $\eta_i\geq |E_{i+1}\cap B_{i+1}|>1$.
Therefore, we get $\Delta^A(H) \leq O(k) + \eta(H)O(k)$.

Now we prove equation \eqref{eq:robust_f}.
We define $Q_m = \{i\,|\; \rank(W_i) < m \text{ and } \rank(W_{i+1})\geq m\}$
and $Q = \sum_{m=1}^k |Q_m|$.
We can bound $Q$ using Lemma~\ref{lem:rankdelta}. We have
\begin{align}
\label{eq:Q_estimate}
Q = \sum_{m=1}^k |Q_m|
        = \sum_{i=1}^{\log k} \max\{0, \rank(W_{i}) - \rank(W_{i-1})\}
        \leq \sum_{i=1}^{\log k} (2\Delta^B(W_i) + c_{i+1}-c_{i}),
\end{align}
which is equal to $2\Delta^B(H) + c(H)$.

We bound $\sum_{j=1}^{\log k} \rank(W_j)$ as a function of $Q$ and $\eta(H)$
\begin{align}
\label{eq:robust_rank_vs_QH}
\sum_{i=1}^{\log k} \rank(W_i) \leq 2 Q \cdot f^{-1}(\frac{a\eta(H)}{Q}).
\end{align}
This relation is proved in Proposition~\ref{prop:robust_rank_vs_QH}
with $a=1$ and, together with \eqref{eq:Q_estimate}, gives us the desired
bound
\begin{align}
\E[\Delta^A(H)] \leq O(1) \Delta^B(H) f^{-1}(\frac{a\eta(H)}{Q}).
\end{align}
\end{proof}

\begin{proposition}
\label{prop:robust_rank_vs_QH}
\[\sum_{i=1}^{\log k} \rank(W_i) \leq 2 Q \cdot f^{-1}(\frac{\eta(H)}{Q}).\]
\end{proposition}
\begin{proof}
We rearrange the sum of ranks in the following way.
We define $L_m = \{i\,|\; \rank(W_i) \geq m\}$,
and $a_{i,m}$, such that $L_m = \bigcup_{i\in Q_m}(i, i+ a_{i,m}]$ for each $m$.
We can write
\begin{align}
\label{eq:robust_rank_vs_aim}
\sum_{i=1}^{\log k} \rank(W_i)
= \sum_{m=1}^{k} |L_m|
= \sum_{m=1}^{k} \sum_{i\in Q_m} a_{i,m}.
\end{align}

On the other hand, we can write
$\eta_i \geq \sum_{m=1}^{\rank(W_i)} |F\cap W_i|$
(Lemma~\ref{lem:error_to_rank})
which allows us to decompose the total prediction error $\eta(H)$ as follows:
\[
\eta(H) \geq \sum_{m=1}^k \sum_{i\in L_m} |F\cap W_i|
= \sum_{m=1}^k \sum_{i\in Q_m} \sum_{j=1}^{a_{i,m}} |F\cap W_{i+j}|.
\]

Let $i^*$ denote the first window such that
$|W_{i^*}| < f(i^*) - f(i^* - 1)$.
If $i + a_{i,m} < i^*$, then
$\sum_{j=1}^{a_{i,m}} |F\cap W_{i+j}| = f(i+a_{i,m}) - f(i) \geq f(a_{i,m})$
by convexity of $f$.
If this is not the case,
we claim that $\sum_{j=1}^{a_{i,m}} |F\cap W_{i+j}| \geq f(a_{i,m}/2)$.
This is clearly true if $i + \lceil a_{i,m}/2\rceil < i^*$.
Otherwise, we have
$\sum_{j=1}^{a_{i,m}} |F\cap W_{i+j}|
\geq \sum_{j=\lceil a_{i,m}/2\rceil}^{a_{i,m}} |F\cap W_{i+j}|
\geq 2^{a_{i,m}/2}$ 
because $i+a_{i,m}\leq k$ and therefore $|W_{i+\lceil a_{i,m}/2\rceil}|
\geq 2^{a_{i,m}/2}$.
By our assumptions about $f$, we have $f(a_{i,m}/2) \leq 2^{a_{i,m}/2}$.

So, we have the following lower bound on $\eta(H)$:
\begin{align}
\label{eq:robust_eta_aim}
\eta(H) \geq \sum_{m=1}^{k} \sum_{i\in Q_m} f(a_{i,m}/2).
\end{align}
By convexity of $f$, this lower bound is smallest if all $a_i^m$
are the same, i.e., equal to the total rank divided by $Q=\sum_m |Q_m|$
and then $f(a_{i,m}/2) = \eta(H)/Q$ for each $i$ and $m$.
Combining \eqref{eq:robust_rank_vs_aim} and \eqref{eq:robust_eta_aim},
we get
\[ \sum_{i=1}^{\log k} \frac{\rank(W_i)}{2}
= \sum_{m=1}^{k} \sum_{i\in Q_m} f^{-1}\big(f\big(\frac{a_{i,m}}{2}\big)\big)
\leq Q\cdot f^{-1}\big(\frac{\eta(H)}{Q}\big).
\qedhere
\]
\end{proof}
\end{fullversion}

\section{Well-separated queries to the predictor}
\label{sec:once_in_a}
\begin{shortversion}
The full version of this section, which can be found in Appendix,
contains a consistent and smooth algorithm for MTS proving
Theorem~\ref{thm:mts} and extends our analysis of $\FR$ to the setting
where the queries to the predictor need to be separated by at least $a$ time
steps, proving Theorem~\ref{thm:FnR_a}.
In MTS, the cost functions usually do not satisfy any Lipschitz property.
Therefore, the difference between the cost of the state reported by the
predictor and the state of the optimal algorithm does not need to be
proportional to their distance.
We show that a state satisfying this property which is close to the predicted
state can be found using a classical technique for design of algorithms
for MTS called {\em work functions}, see \citep{CL96} for reference.
Then, we use the approach of \citet{EmekFKR09}
to interpolate between predictions received at times $t$ and $t+a$.
In the case of caching, the performance of $\FR$ in this regime is the same as if
it has received $a$ incorrect predictions for each prediction error.
Therefore, $\eta$ in its smoothness bound needs to be multiplied by $a$.
\end{shortversion}

\begin{fullversion}
\subsection{MTS}
We consider the setting when we are able to receive a prediction
once every $a$ time steps, for some parameter $a\in \mathbb{N}$.
Without loss of generality,
we assume that $T$ is a multiple of $a$.
In time steps $ia$, where $i=1, \dotsc, T/a$, we receive a prediction
$p_i\in M$. State $p_i$ itself might be very bad, e.g. $\ell_{ia}(p_i)$
might be infinite. We use work functions to see whether there is a more
suitable point nearby $p_i$.

\paragraph{Work functions.}
Consider an MTS on a metric space $M$ with a starting state $x_0\in M$
and a sequence of cost functions
$\ell_1, \dotsc, \ell_T$.
For each time step $t=1, \dotsc, T$ and state $x\in M$, we define
a {\em work function} as
\[ w_t(x)
= \min \big\{d(y_t,x)+ \sum_{i=1}^t d(y_{i-1},y_i)+\ell_i(y_i)\big\},
\]
where the minimum is taken over all $y_0, \dotsc, y_t$ such that
$y_0=x_0$.
In other words, it is the cheapest way to serve all the cost functions
up to time $t$ and end in state $x$.
Work function is a major tool for design of algorithms for MTS and satisfies
the following property.

\begin{observation}
\label{obs:wf}
For any $x,y\in M$ and any time $t$, we have
$w_t(x) \leq w_t(y) + d(y,x)$.
\end{observation}

This holds because one way to serve the cost functions $\ell_1, \dotsc, \ell_t$
is to follow the best solution which ends in state $y$ and then move to $x$.
If $w_t(p_i) = w_t(y) + d(y,p_i)$, we can see that $p_i$ is not a very good
state, since the best solution ending in $p_i$ goes via $y$.
We say that $p_i$ is {\em supported} by $y$.

\paragraph{Algorithm of \citet{EmekFKR09}.}
Algorithm~\ref{alg:emek} was proposed by \citet{EmekFKR09} in the context
of advice complexity. It receives the state of an offline optimum algorithm
every $a$ time steps.
\begin{algorithm2e}
\caption{One cycle of algorithm by \citet{EmekFKR09}}
\label{alg:emek}
        $q_i :=$ reported state of $\OFF$ at time $ia$\;
        $j:=0, c := 0$\;
        \For{$t=ia+1, \dotsc, (i+1)a$}{
                \label{alg:emek_cycle}
                $x_t' := \arg\min_{x\in B(q_i,2^j)} \{d(x, x_{t-1}) + \ell_t(x)\}$\;
                \While{$c+d(x,x_t')+\ell_t(x_t') > 2^j$}{
                        $j := 2*j$\;
                        $x_t' := \arg\min_{x\in B(q_i,2^j)} \{d(x, x_{t-1}) + \ell_t(x)\}$\;
                }
                $x_t:=x_t',c:= c+ d(x,x_t)+\ell_t(x_t)$\;
        }
\end{algorithm2e}

\subsection{Algorithm FtSP}
Given $q_{i-1}$ and $\ell_t$ for $t=(i-1)a+1, \dotsc, ia$,
we define
\[ \wf_i(x) = \min\bigg\{ d(x_{ia},x)+ \sum_{j=(i-1)a+1}^{i\alpha} d(x_{j-1},x_j) + \ell_j(x_j) \bigg\}, \]
where the minimum is taken over $x_{(i-1)\alpha}, \dotsc x_{i\alpha} \in M$
such that $x_{(i-1)a}=q_{i-1}$.
In fact, it is
the work function at the end of an MTS instance with initial state $q_{i-1}$
and request sequence $\ell_t$ for $t=(i-1)a+1, \dotsc, ia$.
Instead of $p_i$, we choose point
\[ q_i = \arg\min_{x\in M} \bigg\{\wf_i(x)\;\bigg|\, \wf_i(x) = \wf_i(p_i)-d(x,p_i)\bigg\}, \]
i.e., the "cheapest" state supporting $p_i$ in $\wf_i$.
After computing $q_i$, we run one cycle of Algorithm~\ref{alg:emek}.
This algorithm which we call ``Follow the Scarce Predictions'' (FtSP),
is summarized in Algorithm~\ref{alg:mts}.

\begin{algorithm2e}
\label{alg:mts}
\caption{FtSP}
\For{$i=0, \dotsc, T/a$}{
        receive prediction $p_i$\;
        use $p_i$ to compute $q_i$\;
        run one cycle of Algorithm~\ref{alg:emek} starting at $q_i$\;
}
\end{algorithm2e}

Let $Q$ denote the best (offline) algorithm which is located at $q_i$ at time
step $ia$ for each $i=1, \dotsc, T/a$.
We have
\[
\cost(Q) = \sum_{i=1}^{T/a} \wf_i(q_i).
\]
We can relate the cost of $Q$ to the prediction error using the following
lemma.
Together with Proposition~\ref{prop:emek}, it gives a bound
\[ \cost(\ALG) \leq O(a)(\OFF + 2\eta),\] 
implying Theorem~\ref{thm:mts}.

\begin{lemma}
\label{lem:ftsp}
Let $\OFF$ be an arbitrary offline algorithm and $o_i$ denote its state
at time $ia$ for $i=1, \dotsc, T/a$.
If Q was computed from predictions $p_1, \dotsc, p_{T/a}$,
we have
\[ \cost(Q) \leq \OFF + 2\eta, \]
where $\eta = \sum_{i=1}^{T/a} d(p_i, o_i)$ is the prediction
error with respect to $\OFF$.
\end{lemma}
\begin{shortversion}
Proof is in Supplementary material.
\end{shortversion}
\begin{proof}
Denote $A_i$ an algorithm which follows the steps of $Q$ until time $ia$ and then follows the steps of $\OFF$.
We have
\[ \cost(A_i) \leq \cost(A_{i-1}) + \wf_i(q_i) - \wf_i(o_i) + d(q_i, o_i) \]
because both $A_i$ and $A_{i-1}$ are at $q_{i-1}$ at time $(i-1)a$,
and
$A_{i-1}$ then travels to $o_i$ paying $\wf_i(o_i)$ while $A_i$ travels to $q_i$ at $ia$ paying $\wf_i(q_i)$ and its
costs after $ia$ will be by at most $d(q_i, o_i)$ larger than the costs of $A_{i-1}$.

By Observation~\ref{obs:wf} and the choice of $q_i$, we have
\[ \wf_i(o_i) \geq \wf_i(p_i) - d(o_i,p_i) =
\wf_i(q_i) + d(q_i,p_i) - d(o_i,p_i).
\]
Combining the two preceding inequalities, we get
\begin{align*}
\cost(A_i) &\leq \cost(A_{i-1}) + \wf_i(q_i) - \wf_i(q_i)\\
	&+ d(o_i,p_i) - d(p_i,q_i) + d(q_i,o_i)\\
&\leq \cost(A_{i-1}) + 2d(o_i,p_i),
\end{align*}
where the last step follows from the triangle inequality.

Since $\OFF=A_0$ and $Q = A_{T/a}$, we have
\[ Q \leq \OFF + 2\sum_{i=1}^{T/a} d(o_i,p_i)
= \OFF + 2\eta. \qedhere\]
\end{proof}

\subsection{Caching}
$\follower$ cannot maintain 1-consistency in this setting.
For the sake of theoretical bound, we can do the following:
We serve the whole input sequence by subsequent phases of $\robust_f$ which is $O(1)$-consistent with $F$ chosen in such a way that the arrivals in $F$
are separated by at least $a$.
We prove the following replacement of Proposition~\ref{prop:robust_rank_vs_QH}.

\begin{proposition}
\label{prop:robust_rank_vs_QH_a}
\[\sum_{i=1}^{\log k} \rank(W_i) \leq 2 Q \cdot f^{-1}\bigg(\frac{a\,\eta(H)}{Q}\bigg).\]
\end{proposition}
\begin{proof}
We rearrange the sum of ranks in the following way.
We define $L_m = \{i\,|\; \rank(W_i) \geq m\}$,
$Q_m = \{i\,|\; \rank(W_i) < m \text{ and } \rank(W_{i+1})\geq m\}$,
and $a_{i,m}$, such that $L_m = \bigcup_{i\in Q_m}(i, i+ a_{i,m}]$ for each $m$.
We can write
\begin{align}
\label{eq:robust_rank_vs_aim_a}
\sum_{i=1}^{\log k} \rank(W_i)
= \sum_{m=1}^{k} |L_m|
= \sum_{m=1}^{k} \sum_{i\in Q_m} a_{i,m}.
\end{align}

On the other hand, we can write
$\eta_i \geq \sum_{m=1}^{\rank(W_i)} |F\cap W_i|$
(Lemma~\ref{lem:error_to_rank})
which allows us to decompose the total prediction error $\eta(H)$ as follows:
\[
\eta(H) \geq \sum_{m=1}^k \sum_{i\in L_m} |F\cap W_i|
= \sum_{m=1}^k \sum_{i\in Q_m} \sum_{j=1}^{a_{i,m}} |F\cap W_{i+j}|.
\]

Let $i^*$ denote the first window such that
$\frac{|W_{i^*}|}{a} < f(i^*) - f(i^* - 1)$.
If $i + a_{i,m} < i^*$, then
$\sum_{j=1}^{a_{i,m}} |F\cap W_{i+j}| = f(i+a_{i,m}) - f(i) \geq f(a_{i,m})$
by convexity of $f$.
If this is not the case,
we claim that $\sum_{j=1}^{a_{i,m}} |F\cap W_{i+j}| \geq a^{-1} f(a_{i,m}/2)$.
\begin{itemize}
\item If $i + \lceil a_{i,m}/2\rceil < i^*$: we have
\[
\sum_{j=1}^{a_{i,m}} |F\cap W_{i+j}|
\geq \sum_{j=1}^{\lceil a_{i,m}/2\rceil} |F\cap W_{i+j}|
\geq f(a_{i,m}/2).
\]
\item Otherwise: we have
\[ \sum_{j=1}^{a_{i,m}} |F\cap W_{i+j}|
\geq \sum_{j=\lceil a_{i,m}/2\rceil}^{a_{i,m}} |F\cap W_{i+j}|
\geq \sum_{j=\lceil a_{i,m}/2\rceil}^{a_{i,m}} \frac1a |W_{i+j}|
\geq \frac1a 2^{a_{i,m}/2}
\geq a^{-1} f(a_{i,m}/2)
\]
By our assumptions about $f$ saying that $f(a_{i,m}/2) \leq 2^{a_{i,m}/2}$.
\end{itemize}

So, we have the following lower bound on $\eta(H)$:
\begin{align}
\label{eq:robust_eta_aim_a}
a \eta(H) \geq \sum_{m=1}^{k} \sum_{i\in Q_m} f(a_{i,m}/2).
\end{align}
By convexity of $f$, this lower bound is smallest if all $a_i^m$
are the same, i.e., equal to $a\eta(H)$ divided by $Q=\sum_m |Q_m|$
and then $f(a_{i,m}/2) = a\eta(H)/Q$ for each $i$ and $m$.
Combining \eqref{eq:robust_rank_vs_aim_a} and \eqref{eq:robust_eta_aim_a},
we get
\[ \sum_{i=1}^{\log k} \frac{\rank(W_i)}{2}
= \sum_{m=1}^{k} \sum_{i\in Q_m} f^{-1}\big(f\big(\frac{a_{i,m}}{2}\big)\big)
\leq Q\cdot f^{-1}\big(\frac{a\eta(H)}{Q}\big) \qedhere
\]
\end{proof}

Using the proposition above in Equation~\eqref{eq:robust_rank_vs_QH}
in the proof of Lemma~\ref{lem:robust}
gives us the following smoothness bound:
\begin{lemma}
Denote $X_i = H_{i-1} \cup H_i^- \cup H_i$. During the phase $H_i$, $\robust_f$
receiving at most one prediction in $a$ time steps incurs the cost
\[ \E[\Delta^A(H_i)] \leq O(1) f^{-1}\bigg(\frac{a \eta(H_i)}{\Delta^B(X_i)}\bigg)
	\Delta^B(X_i).
\]
\end{lemma}

Theorem~\ref{thm:FnR_a} follows from
summation of the bound above over all phases of $\robust$
and concavity of $f^{-1}$, as in proof of Theorem~\ref{thm:FnR}.
\end{fullversion}

\section{Experiments}
\label{sec:Experiments}
We perform an empirical evaluation of our caching algorithm $\FR$
on the same datasets and with the same predictors as the previous
works \citep{LykourisV21,AntoniadisCE0S20,Im0PP22}.
We use the following datasets.

\begin{itemize}
\item BrightKite dataset \citep{BKpaper} contains data from a certain social
network. We create a separate caching instance from the data of each user,
interpreting check-in locations as pages.
We use it with cache size $k=10$
and choose instances corresponding to the first 100 users with the longest check-in
sequences requiring at least 50 page faults in the optimal policy.
\item CitiBike dataset contains data about bike trips in a bike sharing
platform \citet{citi}. We create a caching instance from each month in 2017,
interpreting starting stations of the trips as pages, and trimming length of
each instance to 25~000.
We use it with cache size $k=100$.
\end{itemize}

Some of the algorithms in our comparison use next-arrival predictions
while $\FR$ uses action predictions that can be generated from
next-arrival predictions. Therefore, we use predictors which predict the next
arrival of the requested page and convert it to action predictions.
This process was used and described by \citet{AntoniadisCE0S20}
and we use their implementation of the predictors.
Our algorithm is then provided limited access to the resulting action predictions
while the algorithm of \citet{Im0PP22} has limited access to the original
next-arrival predictions.

\begin{itemize}
\item Synthetic predictions: compute the exact next arrival time computed from
the data and add noise to this number. 
This noise comes from a log-normal distribution with the mean parameter $\mu=0$ and
the standard
deviation parameter $\sigma$. We use $\sigma \in [0, 50]$.
\item PLECO predictor proposed by \citet{PLECO}: This model estimates
the probability $p$ of a page being requested in the next time step
and we interpret this as a prediction that the next arrival of this page
will be in $1/p$ time steps. The model parameters were fitted to BrightKite
dataset and not adjusted before use on CitiBike.
\item POPU -- a simple predictor used by \citet{AntoniadisCE0S20}:
if a page appeared in $p$ fraction of the previous requests,
we predict its next arrival in $1/p$ time steps.
\end{itemize}

In our comparison, we include the following algorithms:
offline algorithm $\belady$ which we use to compute the optimal number of page
faults OPT,
standard online algorithms LRU and Marker~\citep{FiatKLMSY91},
ML-augmented algorithms using next arrival predictions
L\&V \citep{LykourisV21}, LMark and LnonMark \citep{Rohatgi20},
FtPM which, at each step, evicts an unmarked page with
the furthest predicted next arrival time, and
algorithms for action predictions FtP and T\&D \citep{AntoniadisCE0S20}.
We use the implementation of all these algorithms published by
\citet{AntoniadisCE0S20}.
We implement algorithm AQ \citep{Im0PP22}
and our algorithm $\FR$.

\textbf{Notes on implementation of $\FR$.} We follow the recommendations
in Section~\ref{sec:FnR} except that
$\follower$ switches to $\robust$ whenever its cost is $\alpha=1$ times higher
compared to $\belady$ in the same period.
With higher $\alpha$, the performance of $\FR$ approaches FtP on the
considered datasets.
With $k=10$ (BrightKite dataset), we use $F=[1,6,9]$
corresponding to $f(i) = i$. Note that, with such small $k$,
polynomial and exponential $f$ would also give a very similar $F$.
With $k=100$ (CitiBike dataset), we use exponential
$f(i) = 2^{i+1}-1$.
With $a$-separated queries,
$\follower$ uses LRU heuristic when prediction is unavailable,
and $\robust$ ignores $F$, querying the predictor at each page fault
separated from the previous query by at least $a$ time steps.

\begin{shortversion}
\begin{figure}
\resizebox{.5\textwidth}{!}{\input{bk-plot.pgf}}
\hspace{-0.015\textwidth}
\resizebox{.515\textwidth}{!}{\input{bk-tot.pgf}}
\caption{BrightKite dataset with Synthetic predictor, standard deviation at most
 0.003 and 300 resp.}
\label{fig:exp_bk}
\end{figure}
\end{shortversion}

\begin{fullversion}
\begin{figure}
\hfill\resizebox{.65\textwidth}{!}{\input{bk1.pgf}}\hfill\ %
\caption{BrightKite dataset with Synthetic predictor: competitive ratio}
\label{fig:exp_bk}
\end{figure}

\begin{figure}
\hfill\resizebox{.75\textwidth}{!}{\input{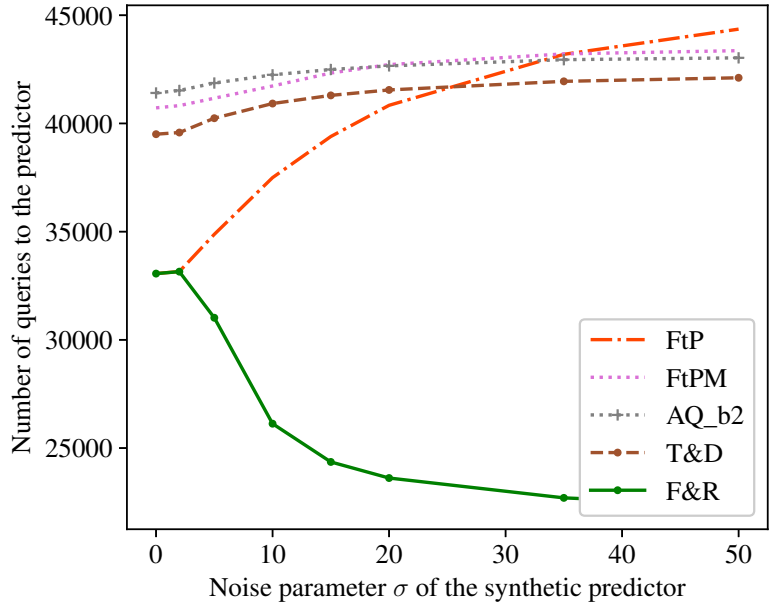}}\hfill\ %
\caption{BrightKite dataset with Synthetic predictor: number of used
predictors}
\label{fig:exp_bk_count}
\end{figure}
\end{fullversion}

\textbf{Results.}
Figures \ref{fig:exp_bk} and \ref{fig:exp_citi} contain averages of 10 independent experiments.
Figure~\ref{fig:exp_bk} shows that
the performance of $\FR$ with high-quality predictions is superior
to the previous ML-augmented algorithms except for FtP which follows the predictions
blindly and is also $1$-consistent.
With high $\sigma$, the performance of T\&D becomes better. This is true
also for $\FR$ with $F=[1..10]$, suggesting that T\&D might be more efficient in using
erroneous predictions.
\begin{shortversion}
The second plot shows the total number of times algorithms query the predictor
over all instances.
\end{shortversion}
\begin{fullversion}
Figure~\ref{fig:exp_bk_count} shows the total number of times algorithms query the predictor
over all instances.
\end{fullversion}
Response to such query is a single page missing from predictor's cache in
the case of $\FR$ and T\&D and next arrival times of $b$ pages in the case of AQ\_$b$.
Note that FtPM is equivalent to the non-parsimonious version of AQ with $b=k$.
$\FR$ makes the smallest number of queries:
with perfect predictions, it makes exactly $\OPT$ queries
and this number decreases with higher $\sigma$
as $\FR$ spends more time in $\robust$.

\begin{shortversion}
\begin{figure}
\begin{tabular}{lcccccccccc}
\toprule
Predictor
	& Marker
	& FtP
	& AQ\_b8
	& FtPM\_a1
	& FtPM\_a5
	& F\&R\_a1
	& F\&R\_a5
	& F\&R\_a20
\\
\midrule
POPU
	& 1.861
	& 1.739
	& 1.782
	& 1.776
	& 1.833
	& 1.800
	& 1.802
	& 1.803
\\
PLECO
	& 1.861
	& 2.277
	& 1.875
	& 1.877
	& 1.867
	& 1.878
	& 1.879
	& 1.879
\\
\bottomrule
\end{tabular}
\caption{Competitive ratios on CitiBike dataset with $k=100$,
standard deviation at most 0.001}
\label{fig:exp_citi}
\end{figure}
\end{shortversion}

\begin{shortversion}
Figure~\ref{fig:exp_citi} shows that $\FR$ performs well in the regime with
$a$-separated queries. While the performance of FtPM with POPU predictor
worsens considerably
towards Marker already with $a=5$, $\FR$ keeps its improvement over Marker
even with $a=20$.
Predictions produced by PLECO seem much less precise as suggested by
FtP with PLECO being worse than Marker and smaller number of such predictions
either improves (AQ, FtPM) or does not affect performance ($\FR$) of
considered algorithms.
Further details of our experimental results are presented in Appendix (Section 5).
\end{shortversion}
\begin{fullversion}
Figure~\ref{fig:exp_citi} shows that $\FR$ performs well in regime with
$a$-separated queries. While the performance of FtPM with POPU predictor
worsens considerably
towards Marker already with $a=5$, the performance of $\FR$ worsens only very
slowly. On CitiBike dataset, it keeps its improvement over
Marker even with $a=20$ (note that we use $k=100$ with this dataset).
Predictions produced by PLECO seem much less precise as suggested by
FtP with PLECO being worse than Marker and smaller number of such predictions
either improves (AQ, FtPM) or does not affect performance ($\FR$) of
considered algorithms.
\end{fullversion}

\begin{fullversion}
\begin{figure}
\begin{center}
\begin{tabular}{llccccccc}
\toprule
Dataset
& Predictor
	& Marker
	& F\&R\_a1
	& F\&R\_a2
	& F\&R\_a3
	& F\&R\_a5
	& F\&R\_a8
	& F\&R\_a20
\\
\midrule
CitiBike &
POPU
	& 1.862
	& 1.800
	& 1.802
	& 1.802
	& 1.802
	& 1.803
	& 1.803
\\
CitiBike &
PLECO
	& 1.862
	& 1.878
	& 1.878
	& 1.878
	& 1.879
	& 1.879
	& 1.879
\\
BrightKite &
POPU
	& 1.333
	& 1.320
	& 1.328
	& 1.332
	& 1.336
	& 1.337
	& 1.341
\\
BrightKite &
PLECO
	& 1.333
	& 1.371
	& 1.374
	& 1.376
	& 1.377
	& 1.378
	& 1.378
\\
\bottomrule
\end{tabular}

\vspace{3ex}
\begin{tabular}{llcccccccc}
\toprule
Dataset &
Predictor
	& T\&D
	& FtP
	& FtPM\_a1
	& FtPM\_a5
	& AQ\_b8
	& L\&V
	& LMark
	& LnonMark
\\
\midrule
CitiBike &
POPU
	& 1.776
	& 1.739
	& 1.776
	& 1.833
	& 1.782
	& 1.776
	& 1.780
	& 1.771
\\
CitiBike &
PLECO
	& 1.847
	& 2.277
	& 1.877
	& 1.866
	& 1.875
	& 1.877
	& 1.876
	& 1.863
\\
BrightKite &
POPU
	& 1.276
	& 1.707
	& 1.262
	& 1.306
	& 1.263
	& 1.262
	& 1.264
	& 1.266
\\
BrightKite &
PLECO
	& 1.292
	& 2.081
	& 1.341
	& 1.337
	& 1.342
	& 1.340
	& 1.337
	& 1.333
\\
\bottomrule
\end{tabular}
\end{center}
\caption{Competitive ratios with predictors POPU and PLECO}
\label{fig:exp_citi}
\end{figure}


\begin{figure}
\begin{center}
\resizebox{.65\textwidth}{!}{\input{bk2.pgf}}
\end{center}
\caption{BrightKite dataset with Synthetic predictor: competitive ratio}
\label{fig:exp_bk2}
\end{figure}

\begin{figure}
\begin{center}
\resizebox{.75\textwidth}{!}{\input{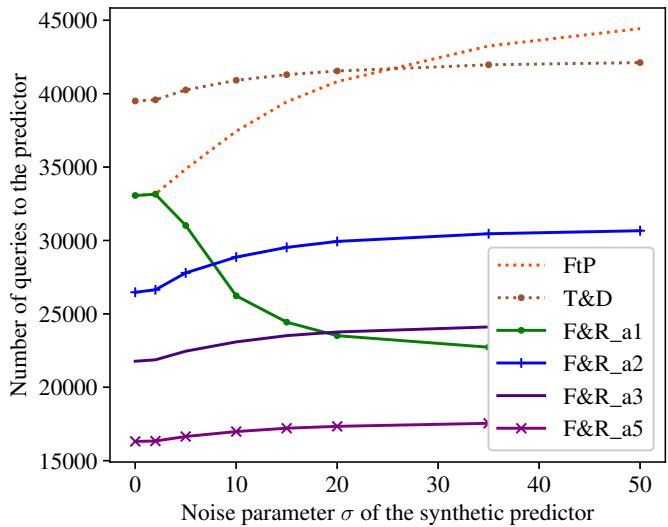}}
\end{center}
\caption{BrightKite dataset with Synthetic predictor: number of used
predictions}
\label{fig:exp_bk2_count}
\end{figure}

Figures \ref{fig:exp_bk2} and \ref{fig:exp_bk2_count}
complements the comparison of $\FR$ to existing ML-augmented algorithms
for paging by including those omitted in Figure~\ref{fig:exp_bk}.
With smaller $\sigma$, it again demonstrates the better consistency of $\FR$.
With higher $\sigma$, $\FR$ achieves performance comparable to both L\&V and
LnonMark, while using a smaller number of predictions.
We have decided not to include LMark because its performance as well as
number of predictions used were almost the same as of LnonMark.
Note that, in the case of both algorithms, the number of used predictions
is equal to the number of clean arrivals and therefore it does not change
with the prediction error.


Figures \ref{fig:exp_bk_a} and \ref{fig:exp_bk_a_count}
shows performance of $\FR$ in regime with $a$-separated queries for
different values of $a$. It shows a significant loss of consistency
already with $a=2$ compared to $a=1$. However, with higher noise parameter
$\sigma$, the difference in performance does not seem large.
In this regime, the focus is on the gap between predictor queries rather
than the total number of queries: $\FR$ queries a predictor at each
page fault separated from previous query by at least $a$ time steps.
However, we decided to include also the plot of the total number of
queries (Figure~\ref{fig:exp_bk_a_count}) because
it shows that
with $\sigma > 20$, $\FR$ with $a=1$ uses a smaller number of predictions
than with $a=2$ and even $a=3$, while maintaining a better performance.
This suggests
that the freedom to choose the right moment for a query might be more important
for the performance than the total number of used predictions.

\begin{figure}
\hfill\resizebox{.75\textwidth}{!}{\input{bk_a.pgf}}\hfill\ %
\caption{BrightKite dataset with Synthetic predictor: competitive ratio}
\label{fig:exp_bk_a}
\end{figure}


\begin{figure}
\begin{center}
\resizebox{.65\textwidth}{!}{\input{bk_a_count.pgf}}
\end{center}
\caption{BrightKite dataset with Synthetic predictor: number of used
predictions}
\label{fig:exp_bk_a_count}
\end{figure}

%
%

\begin{figure}
\hfill\resizebox{.75\textwidth}{!}{\input{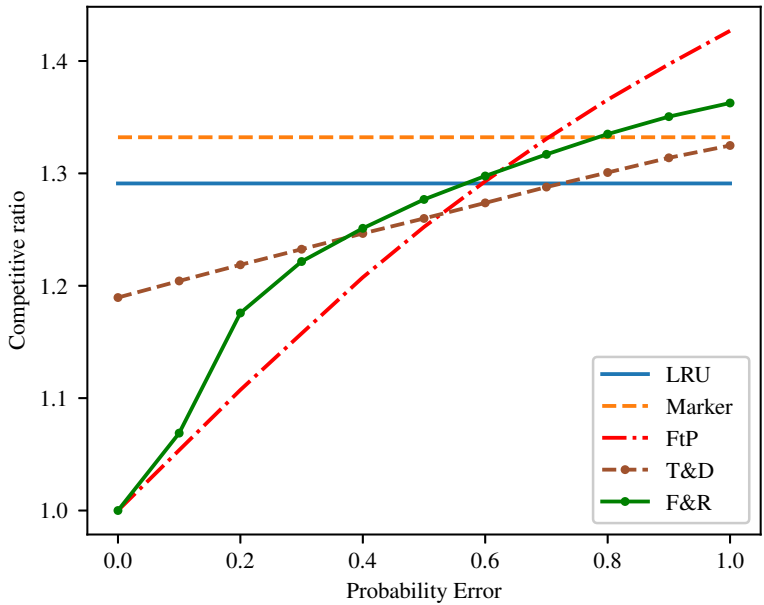}}\hfill\ %
\caption{BrightKite dataset with probabilistic predictor: competitive ratio}
\label{fig:prob_bk_10}
\end{figure}

Figure \ref{fig:prob_bk_10} shows experiments with a probabilistic predictor on the BrightKite dataset. In this setting, we consider a predictor that evicts the page requested furthest in the future with a given probability $1-p$. On the other hand, it evicts a random page with probability $p$. The horizontal axis corresponds to the probability $p$. We can observe that better consistency of our algorithm compared to T\&D is visible for $p$ up to $0.4$.

Each plot and table contains averages of 10 independent experiments.
We have seen standard deviations at most 0.004 in the case of figures
\ref{fig:exp_bk}, \ref{fig:exp_bk2}, \ref{fig:exp_bk_a};
0.0015 for Figure~\ref{fig:exp_citi} on CitiBike dataset and
0.0025 on BrightKite dataset, and
300 for figures
\ref{fig:exp_bk_count}, \ref{fig:exp_bk2_count}, \ref{fig:exp_bk_a_count},
counting numbers of used predictions.
\end{fullversion}

\section{Conclusions}
We present algorithms for MTS and caching with action predictions
working in the setting
where the number of queries or the frequency of querying the predictor
are limited.
We have shown that one can achieve theoretical as well as empirical performance 
comparable to the setting with unlimited access to the predictor,
possibly enabling usage of precise but heavy-weight prediction models
in environments with scarce computational resources.

\begin{shortversion}
\section*{Reproducibility Statement}

The appendix contains a full version of our paper which includes proof
of all the theorems and lemmas. We provide textual description of the
implementation of our algorithm in Section~\ref{sec:Experiments}. The code of
our implementation can be found at
\url{https://github.com/marek-elias/caching/}
\end{shortversion}

\begin{fullversion}
\section{Lower bounds}
\label{sec:LBs}

\subsection{Caching}

Proof of the following proposition can be found in \citep[Theorem 4.4]{BEY98}.

\begin{proposition}[\citep{FiatKLMSY91}]
\label{prop:std_caching_LB}
For any randomized algorithm $\ALG$ for caching there is an input instance
on universe of $k+1$ pages
such that the expected cost of $\ALG$ is more than $\ln k$ times the cost
of the offline optimal solution.
\end{proposition}

For a given algorithm,
it constructs an instance consisting of marking phases, each
with a single clean page such that the optimal algorithm pays 1
and the online algorithm pays at least $\ln k$.

Imagine an algorithm receiving at most $0.5\OPT$ predictions during
this instance. Then, there must be at least $0.5\OPT$ phases during
which the algorithm receives no prediction. Its cost is at least $\ln k$
in each such phase, giving total cost $0.5\OPT\ln k$.

Theorem~\ref{thm:FnR_LB} is implied by the following more general statement
with $c=1$ and $d=0$.

\begin{theorem}
Let $c \geq 1$ and $d \geq 0$ be constants.
Any $(cf^{-1}(\eta) + d)$-smooth algorithm for caching with action predictions
has to use at least $f(c^{-1}\ln k - d)\OPT$ predictions.
\end{theorem}
\begin{proof}
Consider a fixed algorithm accepting action predictions.
Choose $T$ long enough, an arbitrary prediction for each time step
$t=1, \dotsc, T$,
and give them to the algorithm at time $0$. Having the predictions already,
this algorithm becomes a standard randomized algorithm which does not
use any further predictions. We use Proposition~\ref{prop:std_caching_LB}
to generate an instance such that $\E[\ALG] \geq \OPT \ln k$,
where $\ALG$ denotes the cost of the algorithm with predictions
generated in advance.
It is clear that these predictions, generated before the adversary has
chosen the input instance, are useless, not helping the algorithm
to surpass the worst-case bounds.
However, since the universe of pages has size only $k+1$, each of the
predictions can differ from an optimal algorithm by at most one page.

If we want to have $\frac{\E[\ALG]}{\OPT} \leq cf^{-1}(\frac{\eta}{\OPT}) + d$,
then we need
\[ \frac{\eta}{\OPT} \geq f\bigg(\frac{\E[\ALG]}{c\OPT} - d\bigg)
	> f\bigg(\frac{\ln k}{c} - d\bigg). \]
Since every prediction has error at most 1,
we need to receive at least $\eta \geq f(c^{-1}\ln k -d)\OPT$ predictions.
\end{proof}

\subsection{MTS}

\citet{AntoniadisCE0S20} showed the following lower bound on smoothness
of algorithms for general MTS with action predictions.

\begin{proposition}[\citet{AntoniadisCE0S20}]
For $\eta\ge 0$ and $n\in\N$, every deterministic (or randomized) online
algorithm for MTS on the $n$-point uniform metric with access to an action
prediction oracle with error at least $\eta$ with respect to some
optimal offline algorithm has competitive ratio
$\Omega\left(\min\left\{\alpha_n, 1+\frac{\eta}{\OPT}\right\}\right)$, where
$\alpha_n=\Theta(n)$ (or $\alpha_n=\Theta(\log n)$) is the optimal competitive
ratio of deterministic (or randomized) algorithms without prediction.
\end{proposition}

We use this proposition to prove the following theorem from which
Theorem~\ref{thm:mts_LB} directly follows.

\begin{theorem}
For $\eta\ge 0$ and $n\in\N$, every deterministic (or randomized) online
algorithm for MTS on the $n$-point uniform metric with access to an action
prediction oracle at most once in $a$ time steps with error at least $\eta$ with
respect to some
optimal offline algorithm has competitive ratio
$\Omega\left(\min\left\{\alpha_n, 1+\frac{a\eta}{\OPT}\right\}\right)$,
where
$\alpha_n=\Theta(n)$ (or $\alpha_n=\Theta(\log n)$) is the optimal competitive
ratio of deterministic (or randomized) algorithms without prediction.
\end{theorem}
\begin{proof}
We extend the $(n-1)$-point uniform metric from the proposition above by a
single point $p_\infty$
whose cost will be $+\infty$ at each time step, ensuring the optimal
algorithm will never be located there.
Consider a fixed algorithm and
a predictor producing at most one prediction in $a$ time steps
with the total prediction error $\eta$. By issuing prediction $p_\infty$ in
all missing time steps, we complete predictions for each time step with
error at least $\eta' \geq a\eta$.

By proposition above, the algorithm with completed predictions has
competitive ratio at least
\[ \Omega\bigg(\min\big\{\alpha_{n-1}, 1+\frac{\eta'}{\OPT}\big\}\bigg)
\geq \Omega\bigg(\min\big\{\alpha_n, 1+\frac{a\eta}{\OPT}\big\}\bigg),
\]
since $\alpha_n$ and $\alpha_{n-1}$ differ by at most a constant factor.
\end{proof}
\end{fullversion}

\begin{fullversion}
\section{FitF oracle}
\label{sec:FitF_oracle}
In this section we work with a predictor which tells us which page in
our current cache will be requested furthest in the future, we call it
a FitF page.
Note that this is not the same as the predictions considered in
Section~\ref{sec:FnR}, where we receive a page not present in $\belady$'s cache.
$\belady$ evicts a FitF page from its current cache content which may be
different from the FitF page from the current cache content of our algorithm.
Prediction error is the total number of times the predictor reports an incorrect
FitF page.

We split our algorithm into
$\follower$ and $\robust$ part. The $\follower$
(Algorithm~\ref{alg:follower_oracle}),
checks at each page fault whether $\belady$ starting at the same
time with the same cache content also has a page fault.
If yes, it evicts a page reported by the predictor. Otherwise,
it switches to the $\robust$ part (Algorithm~\ref{alg:robust_oracle}).

\begin{algorithm2e}
\caption{Follower with FitF oracle}
\label{alg:follower_oracle}
$P:=$ starting cache content\;
\ForEach{pagefault}{
        Compute $\belady$ for the sequence from the beginning of this execution
        starting with $P$\;
        \uIf{$\belady$ has page fault as well}{
                $p:=$ page in the current cache chosen by the predictor\;
                evict $p$\;
        }
        \Else{
                Run one phase of Algorithm~\ref{alg:robust_oracle} starting
                with the current cache content\;
        }
}
\end{algorithm2e}

\begin{lemma}
\label{lem:follower_oracle}
Consider one execution of Algorithm~\ref{alg:follower_oracle},
denoting $\sigma$ the request subsequence and $\varphi$ the number
of incorrect predictions received during this execution.
Algorithm~\ref{alg:follower_oracle} pays the same cost as
$\belady$ serving $\sigma$ and starting with cache content $P$.
There is a tie-breaking rule for $\belady$ such that the cache contents
of both algorithms after processing $\sigma$ differ in at most $\varphi$ pages.
\end{lemma}
\begin{proof}
Whenever the algorithm has a page fault and $\belady$ not, the execution
of Algorithm~\ref{alg:follower_oracle} terminates. Therefore, both
algorithms have the same cost during the execution.

Denote $A$ and $B$ the cache contents of our algorithm and $\belady$
respectively.
We choose the following tie-breaking rule for $\belady$:
whenever the algorithm evicts $p\in A\cap B$ which is no more requested
in $\sigma$, $\belady$ evicts $p$ as well.
The size of $A\setminus B$ increases only when
the algorithm evicts a predicted page $p\in A\cap B$ and $\belady$
evicts a different page $q\in A\cap B$.
This can happen only if the next request of $p$ comes earlier than $q$
by the tie-breaking rule above.
Since $p,q\in A$, the oracle made a prediction error.
\end{proof}

$\robust$ part (Algorithm~\ref{alg:robust_oracle}) uses a parameter $b$
which controls the number of predictions used during its execution.
It runs for a duration of a single marking phase split into $\log k$
windows, as in Section~\ref{sec:FnR}, making sure that the number
of predictions received in each window is at most the number of clean pages
received so far. Evictions of random unmarked pages are used at page faults
with no available prediction. At the end, it loads all marked pages.
This is to ensure that the difference between the optimal and algorithm's
cache content can be bounded by the cost of the optimal algorithm
during the phase (using Observation~\ref{obs:marking_diverg})
instead of accumulating over repeated executions of $\follower$ and $\robust$.

\begin{algorithm2e}
\caption{Robust with FitF oracle}
\label{alg:robust_oracle}
$P:=$ starting cache content\;
$S:= [t = k-2^j+1\;|\, \text{ for } j=\log k, \dotsc, 0]$\;
$W_i := [S[i], S[i+1]-1]$ for $i=1, \dotsc, \log k+1$
\tcp*{Split the phase into windows}
\ForEach{pagefault at time $t$ during the phase}{
        $c_t:=$ number of clean pages which arrived so far\;
        \uIf{number of received predictions in the phase is less than $bc_t$}{
                \If{number of received predictions in this window is less than $c_t$}{
                $p:=$ page in the current cache chosen by the predictor\;
                evict $p$\;
                }
        }
        \Else{
                evict a random unmarked page\;
        }
}
Once phase has ended, load all marked pages to the cache and run Algorithm~\ref{alg:follower_oracle}\;
\end{algorithm2e}

\begin{lemma}
\label{lem:robust_oracle}
Consider one execution of Algorithm~\ref{alg:robust_oracle}
during which it receives $\varphi$ incorrect predictions.
The expected cost incurred by
Algorithm~\ref{alg:robust_oracle} is at most
$2\Delta^B + 3\varphi(1+b^{-1}\log k)$, where $\Delta^B$ denotes the cost
incurred by $\belady$ starting at the same time with the same cache content.
\end{lemma}
\begin{proof}
There are three kind of page faults:
\begin{enumerate}
\item evicted page is chosen by the predictor
\item requested page was chosen before by the predictor,
	evicted page was chosen at random
\item both evicted and requested pages were chosen at random
\end{enumerate}
In the worst case, we can assume that
once we run
out of budget for predictions, all incorrectly evicted pages
are requested in page faults of type 2 and returned to the cache.
Now, let $g$ denote the number of pages evicted due to correct predictions
-- they are not going to be requested in this phase
anymore (Observation~\ref{obs:marking_FitF_err}).
All other evicted pages are chosen uniformly at random among unmarked pages
which were not evicted due to correct predictions.
So, until another batch of page faults of type 1, we have only
page faults on arrivals and the probability of a page fault on arrival $a$
is at most
\[ \frac{c_a - g}{k-(a-c_t)-g},\]
where $c_a$ is the number of clean pages until arrival $a$ and
$k-(a-c_t)$ is the number of unmarked pages, at most $g$ of them
were evicted due to correct predictions.

We count the number of page faults in window $i$ for $i=1,\dotsc, \log k +1$.
We denote $m_i$ the number of page faults of type 1 and
resulting into eviction of $g_i$ correctly predicted pages.
Then, by our assumption, we have $m_i-g_i$ page faults of type 2.
The expected number of page faults of type 3 depends on when do types 1 and 2
happen. In the worst case, they all happen in the beginning of $W_i$
as well as all arrivals of clean pages.
We consider three cases.

\paragraph{Case A.} Prediction budget was not depleted, there were
only evictions of type 1.
\[ \Delta^A(W_i) = m_i = \varphi_i + g_i.\]

\paragraph{Case B.} There were $m_i = c_{i+1}$ predictions during $W_i$
and we have $\varphi_i = c_{i+1} - g_i$. After page faults of type 2,
there are at most $c_{i+1}-g_i$ randomly chosen unmarked pages evicted.
Therefore, the expected number of page faults of type 3 is at most
\begin{align*}
\sum_{a\in W_i} \frac{c_{i+1}-g_a}{k-(a-c_{i+1})-g_a}
        \leq \sum_{a\in W_i} \frac{c_{i+1}-g_i}{k-(a-c_{i+1})-g_i}\\
        \leq \sum_{a\in W_i} \frac{c_{i+1}-g_i}{k-a}
        \leq \frac{k}{2^{i}}\cdot \frac{c_{i+1}-g_i}{k/2^{i}}
        = c_{i+1} - g_i.
\end{align*}
Therefore, counting evictions of types 1, 2, and 3, we have
\[
\Delta^A(W_i) \leq (\varphi_i + g_i) + \varphi_i + (c_{i+1}-g_i)
	\leq g_i + 3\varphi_i.
\]

\paragraph{Case C.} There were $bc_{i+1}$ predictions since the beginning
of the phase. We have $m_i\leq c_{i+1}$ and
$c_{i+1}-g_i \leq\frac1b(bc_{i+1}-g_i) \leq \frac1b \varphi$
where $\varphi$ is the total number of incorrect predictions received
since the beginning of the phase. We have
\[
\Delta^A(W_i) \leq c_{i+1} + (c_{i+1}-g_i) + (c_{i+1}-g_i)
	\leq g_i + 3(c_{i+1} - g_i),
\]
which is at most $g_i + 3\varphi/b$.

Now, the sum of costs over all the windows is at most
\[ \sum_i g_i + \sum_i 3\varphi_i + \sum_i 3\varphi/b + c
	\leq 2c + 3\varphi + \frac{3\varphi}{b} \log k,
\]
where $c=\sum_i c_i \leq \Delta^B$, because we consider $\belady$
starting with the same cache content as the algorithm which does not
contain the clean pages.
\end{proof}

\begin{theorem}
Let $b \in \{1, \dotsc, \log k\}$ be a parameter.
During a request sequence with optimum cost $\OPT$, our algorithm
receives at most $O(b)\OPT$ predictions and its expected cost is always bounded by
$O(\log k) \OPT$.
If only $\varphi$ predictions are incorrect,
its expected cost is at most
\[ \bigg(2 + \frac{\varphi}{\OPT}(4 + 3b^{-1}\log k)\bigg) \OPT.\]
Moreover, if $\varphi=0$, its cost is equal to $\OPT$.
\end{theorem}
\begin{proof}
We split the time horizon into intervals corresponding to executions
of $\follower$ and $\robust$. For each interval $i$, we denote
$\varphi_i$ the number of received incorrect predictions,
$\Delta^B_i$
the cost incurred by $\belady$ started with the same content as our algorithm
and $\Delta^O_i$ the cost incurred by the optimal solution during interval $i$.
We denote $F$ the set of intervals
during which $\follower$ was executed and $R$ the set of intervals
during which $\robust$ was executed.
We also define $0\in R$ an empty interval in the beginning
of the request sequence with $\Delta^O_0 = \Delta^B_0 = 0$.

In order to prove bounds on robustness and number of used predictions,
we provide relations between $\Delta^B_i$ and $\Delta^O_i$
independent of $\varphi$.
For each $i \in F$, we have $i-1\in R$. Interval $i-1$ is a marking phase
and $\robust$ has all marked pages in the cache at the end (Lemma~\ref{lem:robust_oracle}).
By Observation~\ref{obs:marking_diverg}, the starting cache content of
$\follower$ in interval $i$ differs from optimal cache
in at most $\Delta^O_{i-1}$ pages.
Therefore, we have
\begin{equation}
\label{eq:oracle_DBF}
\Delta^B_i \leq \Delta^O_i + \Delta^O_{i-1} \quad\forall i\in F.
\end{equation}
For each $i\in R$, we have $i-1\in F$ and $i-2\in R$.
By Observation~\ref{obs:lazy_diverg}, the difference between the cache of $\follower$ and optimum increases during interval $i-1$ by at most
$\Delta^O_{i-1}$.
Since the starting cache of $\follower$ in interval $i-1$ differs from
optimal in $\Delta^O_{i-2}$ pages, the starting cache of $\robust$ in
interval $i$ differs from optimum by at most $\Delta^O_{i-2}+\Delta^O_{i-1}$.
Therefore, we have
\begin{equation}
\label{eq:oracle_DBR}
\Delta^B_i \leq \Delta^O_i + \Delta^O_{i-1} + \Delta^O_{i-2} \quad\forall i\in R.
\end{equation}

Using equations~\eqref{eq:oracle_DBF} and \eqref{eq:oracle_DBR},
we can bound the number of used predictions as
\[
\sum_{i\in F} \Delta^B_i + \sum_{i\in R} b\Delta^B_i
    \leq 3b\OPT.
\]

Since $\varphi \leq \Delta^B_i$, for $i\in F$, and $\varphi_i \leq b\Delta^B_i$
for $i\in R$, we have the following robustness bound:
\begin{align*}
\ALG &\leq \sum_{i\in F} \Delta^B_i
    + \sum_{i\in R} \big(\Delta^B_i + \varphi_i(3+3b^{-1}\log k)\big)\\
&\leq \sum_{i\in F} \Delta^B_i
    + \sum_{i\in R} \Delta^B_i (1+b)(3+3b^{-1}\log k)\\
&\leq \OPT \cdot O(\log k),
\end{align*}
where the last inequality follows from
\eqref{eq:oracle_DBF}, \eqref{eq:oracle_DBR}, and $b\leq \log k$.

Now, we analyze smoothness. We can bound $\Delta^B_i - \Delta^O_i$ by
the difference between optimal and algorithm's cache in the beginning
of the interval $i$.
This is at most $\varphi_{i-1}$ for each $i\in R$
(Lemma~\ref{lem:follower_oracle}) and at most $\Delta^O_{i-1}$
for each $i\in F$ by \eqref{eq:oracle_DBF}.
Lemmas~\ref{lem:follower_oracle} and \ref{lem:robust_oracle}
imply
\begin{align*}
\ALG &\leq \sum_{i\in F} \Delta^B_i
    + \sum_{i\in R} \big(2\Delta^B_i + \varphi_i(3+3b^{-1}\log k)\big)\\
&\leq \sum_{i\in F} (\Delta^O_i + \Delta^O_{i-1})
    + \sum_{i\in R} \big(\Delta^O_i + \varphi_{i-1} + \varphi_i(3+3b^{-1}\log k)\big)\\
&\leq 2\OPT + \varphi(4+3b^{-1}\log k).
\end{align*}

$1$-consistency of our algorithm can be seen from the fact that
each execution of $\robust$ is triggered by an incorrect prediction.
Therefore, with perfect predictions, only $\follower$ is used and
behaves the same as $\belady$.
\end{proof}
\end{fullversion}

\setcitestyle{numbers}
\bibliography{references,draft}
\bibliographystyle{abbrvnat}

\newpage
\appendix
\onecolumn

\begin{fullversion}
\section{Computations for Table 1}
\label{app:pred_numbers}
In this section, we present the computations for the numbers of predictions we
obtained in Table ~\ref{tab:f_smooth}.

For $f(i)= i$ and $f(i) = i^2$, we have
$f(\log k)$ equal to $\log k$ and $\log^2 k$ respectively.

For $f(i) = 2^i-1$, we identify
the first window $i$ longer than $f(i)-f(i-1)$.
Note that the length of window $i$ is $k/2^i = 2^{\log k - i}$
and this is equal to the sum of lengths of the windows $j > i$.
The total number of predictions used will be therefore $f(i) + 2^{\log k - i}$.
For $i = \log \sqrt{k}+1$, we have
$f(i) - f(i-1) = 2^i = 2^{\frac12 \log k + 1} >
2^{\log k - i}$. Therefore, we use
$2^i-1 + 2^{\log k -i} \geq 3\sqrt{k}$ predictions in each robust phase.
Since offline optimum has to pay at least $1$ per robust phase,
we use at most $O(\sqrt{k})\OPT$ predictions in total.

For $f(i)=0$, we ask for a prediction at each arrival of a clean page.
The number of queries used will therefore be at most the number of clean
arrivals, which is at most $2\OPT$.

\end{fullversion}

\end{document}